%% file: main.tex
\documentclass[letterpaper]{article} 
\usepackage{aaai2026}  
\usepackage{times}  
\usepackage{helvet}  
\usepackage{courier}  
\usepackage[hyphens]{url}  
\usepackage{graphicx} 
\usepackage{xcolor}
\usepackage{adjustbox}
\usepackage{natbib}  
\usepackage{caption} 
\usepackage{subcaption}
\usepackage{algorithm}
\usepackage{algorithmic}
\usepackage{tikz}
\usepackage{enumitem}
\usepackage{microtype}
\usepackage{graphicx}
\usepackage{booktabs} 
\usepackage{longtable}
\usepackage{multirow}
\usepackage{graphicx}
\usepackage[table,xcdraw]{xcolor}
\usepackage{placeins}
\usepackage{amsmath}
\usepackage{amssymb}
\usepackage{mathtools}
\usepackage{amsthm}
\usepackage[capitalize,noabbrev]{cleveref}
\usepackage[T1]{fontenc}
\theoremstyle{plain}
\newtheorem{lemma}{Lemma}

\theoremstyle{definition}

\theoremstyle{remark}

\usepackage[textsize=tiny]{todonotes}

\usepackage{soul}

\newcommand{\tightTable}{
  \setlength{\aboverulesep}{0.5pt}
  \setlength{\belowrulesep}{0.5pt}
  \renewcommand{\arraystretch}{0.95}
  \setlength{\tabcolsep}{3pt}
}

\usepackage{newfloat}
\usepackage{listings}
\DeclareCaptionStyle{ruled}{labelfont=normalfont,labelsep=colon,strut=off} 
\floatstyle{ruled}
\newfloat{listing}{tb}{lst}{}
\floatname{listing}{Listing}
\newcommand{\ourmethod}{\textsc{LoReTTA}}

\newcommand{\Call}[2]{\textsc{#1}(#2)}

\title{\textsc{LoReTTA}: A Low Resource Framework To Poison Continuous Time Dynamic Graphs}
\author{
    Himanshu Pal\equalcontrib,
    Venkata Sai Pranav Bachina\equalcontrib,
    Ankit Gangwal,
    Charu Sharma
}
\affiliations{
    International Institute of Information Technology, Hyderabad\\

}

\usepackage{bibentry}

\begin{document}

\maketitle

\begin{abstract}

Temporal Graph Neural Networks~(TGNNs) are increasingly used in high-stakes domains, such as financial forecasting, recommendation systems, and fraud detection. However, their susceptibility to poisoning attacks poses a critical security risk. We introduce \textbf{\ourmethod} (\textbf{Lo}w \textbf{Re}source \textbf{T}wo-phase \textbf{T}emporal \textbf{A}ttack), a novel adversarial framework on Continuous-Time Dynamic Graphs, which degrades TGNN performance by an average of \textbf{29.47\%} across 4 widely benchmark datasets and 4 State-of-the-Art (SotA) models.


 \ourmethod~operates through a two-stage approach: (1) sparsify the graph by removing high-impact edges using any of the 16 tested temporal importance metrics, (2) strategically replace removed edges with adversarial negatives via \ourmethod's novel degree-preserving negative sampling algorithm. Our plug-and-play design eliminates the need for expensive surrogate models while adhering to realistic unnoticeability constraints. \ourmethod~degrades performance by upto \textbf{42.0\%} on MOOC, \textbf{31.5\%} on Wikipedia, \textbf{28.8\%} on UCI, and \textbf{15.6\%} on Enron. \ourmethod~outperforms 11 attack baselines, remains undetectable to 4 leading anomaly detection systems, and is robust to 4 SotA adversarial defense training methods, establishing its effectiveness, unnoticeability, and robustness.



\end{abstract}

\begin{links}
    \link{Code}{https://github.com/ansh997/LoReTTA}
    
\end{links}


\input{sections/sec-1-Introduction}
\input{sections/sec-2-RelatedWorks}

\input{sections/sec-2.5-preliminaries}

\input{sections/sec-3-Method}
\FloatBarrier
\input{sections/sec-4-Experiments}

\input{sections/sec-5-Discussion}

\input{sections/sec-6-Conclusion}

\bibliography{reference}

\appendix
\input{sections/section-8-appendix}
\end{document}

%% file: sections/sec-1-Introduction.tex
\section{Introduction}

Temporal Graph Neural Networks~(TGNNs) are increasingly deployed in real-world systems such as social media analysis~\cite{deng2019learning, fan2019graph}, transportation modeling~\cite{jiang2022graph, yu2017spatio}, recommendation engines~\cite{gao2022graph, wu2022graph}, outbreak tracking~\cite{cencetti2021digital, so2020visualizing}, and knowledge graph reasoning~\cite{cai2022temporal}. These applications often rely on modeling evolving relationships over time using dynamic graphs, making the robustness of TGNNs a crucial concern.

Dynamic graphs can be broadly categorized as Discrete-Time Dynamic Graphs (DTDGs) and CTDGs. DTDGs represent dynamic graphs as a sequence of static snapshots, while CTDGs model them as a continuous stream of timestamped interactions. CTDGs are considered more expressive for modeling real-world systems, as they capture fine-grained temporal dependencies that DTDGs may miss~\cite{zheng2023decoupled, ennadir2024expressivityrepresentationlearningcontinuoustime}. 
\begin{figure}[!ht]
    \centering
    \includegraphics[width=\columnwidth,
                 trim=0 150mm 0 0,  
                 clip]{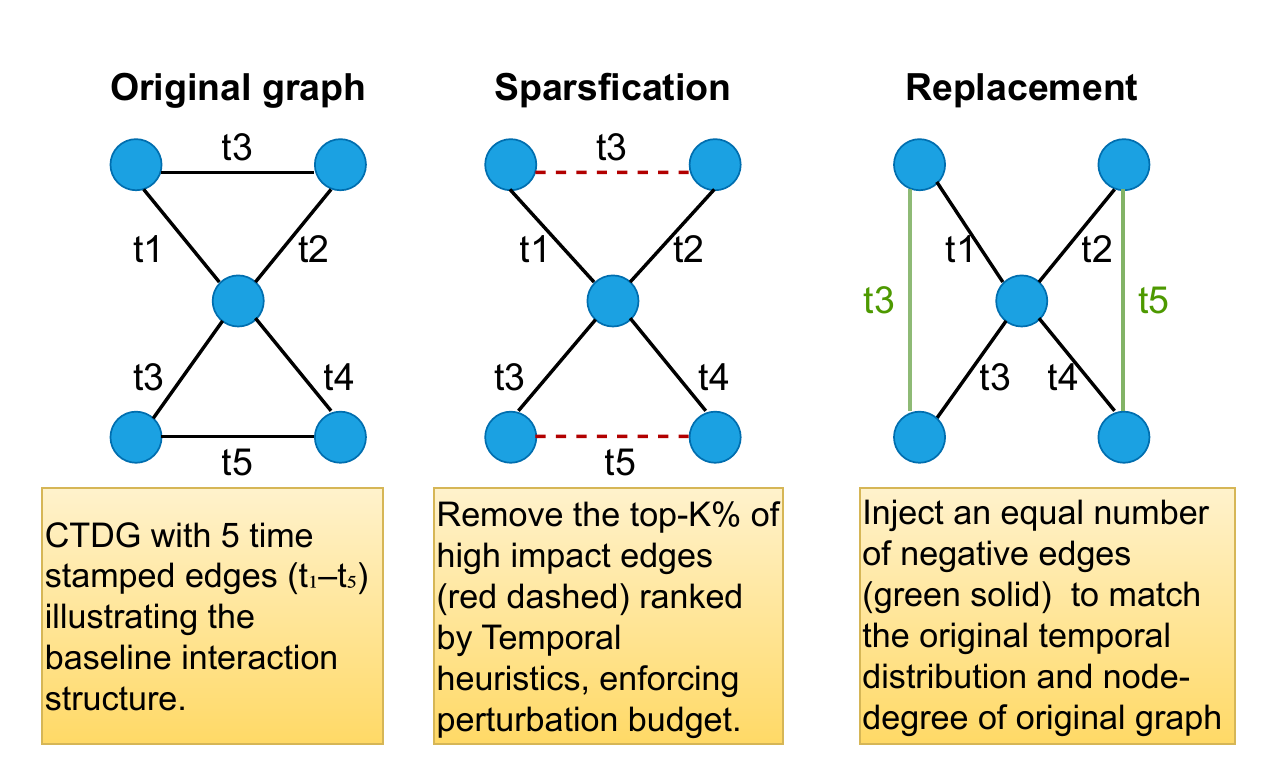}
    \caption{Overview of the \ourmethod~attack framework. 
    }
    \label{fig:arch_diagram}
\end{figure}

Recent work~\cite{lee2024spear} has demonstrated that TGNNs are susceptible to adversarial poisoning attacks, i.e small but carefully crafted perturbations to the training graph that can significantly degrade performance. This threat becomes especially critical in high-stakes domains such as fraud detection or traffic forecasting, where reliability is paramount. 

While adversarial attacks have been extensively studied for static graphs~\cite{zugner2018adversarial, bojchevski2019adversarial, dai2018adversarial, ma2020towards, zügner2024adversarialattacksgraphneural}, their temporal counterparts~(CTDGs) are less explored. The challenges are non-trivial: CTDGs evolve over time, requiring the attacker to perturb edges with precise temporal coordination. TGNNs further complicate this by maintaining memory states that propagate historical information and may dilute isolated attacks. Moreover, the attacker must decide \textit{when} and \textit{where} to insert perturbations to maximize the impact, all the while adhering to unnoticeability constraints~(cf \S\ref{sec:unnoticeability_constraints}).

The SotA work, i.e., Spear and Shield~\cite{lee2024spear}, introduces an adversarial poisoning framework for TGNNs, but it carries significant practical limitations. Specifically, it relies on training a surrogate model to estimate gradients, an approach that is \textbf{computationally expensive for large-scale datasets}. Furthermore, it assumes the attacker has complete access to the entire dataset~(train, validation, and test) and the ability to make arbitrary changes. In real-world scenarios, such assumptions rarely hold in the context of an attacker. Attackers often only have access to the training portion of the graph and must operate under realistic, limited knowledge and manipulation constraints.

To address these limitations, we propose a novel adversarial poisoning framework for TGNNs, that we call \textbf{\ourmethod}~(\textbf{Lo}w \textbf{Re}source \textbf{T}wo-phase \textbf{T}emporal \textbf{A}ttack). \ourmethod~does not rely on surrogate models, and it assumes access to only the training portion of the data. It adheres to 4 practical unnoticeability constraints (cf.~\S\ref{sec:unnoticeability_constraints}) that preserve the structural and temporal plausibility of the poisoned graph. At its core, \ourmethod~computes scores for edges~(using any of our 16 tested heuristics) across timestamps to quantify evolving influence, and removes edges with high temporal influence; thereby maximizing disruption to the TGNN’s representation learning. These edges are then replaced with strategically selected adversarial negatives through \ourmethod's novel negative sampling algorithm designed to preserve the original graph’s edge density and local structural cues.\\
Our main contributions are summarized below:
\begin{enumerate}
    \item We propose \ourmethod, a novel adversarial poisoning framework that operates by removing high-influence edges and replacing them with adversarial negatives using \ourmethod's novel negative sampling algorithm.
    \item We enforce 4 unnoticeability constraints~(cf.\S\ref{sec:unnoticeability_constraints}) that includes graph sparsity, temporal plausibility, node activity windows, and degree preservation to ensure attack stealth.
    \item We demonstrate \ourmethod~(a) outperforms 11 baselines~(cf.\S\ref{sec:attack_perf}) with 29.47\% average degradation across 4 datasets and 4 SotA TGNN models; (b) while evading 4 SotA anomaly detectors~(cf.\S\ref{sec:anomaly_detection}) and (c) resisting 4 adversarial defense methods~(cf.\S\ref{robustness}).
\end{enumerate}

%% file: sections/sec-2-RelatedWorks.tex
\section{Related Work}

\subsection{Temporal Graph Neural Networks (TGNNs)}

Dynamic graphs are commonly categorized into two settings: DTDGs and CTDGs. DTDG approaches discretize a temporal graph into sequential snapshots and apply static graph learning to each frame~\cite{pareja2020evolvegcn, sankar2020dysat}. However, these methods neglect fine-grained temporal continuity, leading to less faithful modeling of real-world dynamic systems.
In contrast, CTDG-based approaches, particularly TGNNs~\cite{trivedi2019dyrep, rossi2020temporal, xu2020inductive, ma2020streaming}, model interactions as continuous temporal streams and capture temporal dependencies directly. These models often incorporate memory networks~\cite{ma2020streaming, rossi2020temporal}, sequential encoders~\cite{cong2023we}, and temporal message-passing schemes, allowing more expressive modeling of evolving systems. TGNNs have shown success in modeling complex phenomena such as user-item interactions, social dynamics, and temporal knowledge graphs~\cite{ennadir2024expressivityrepresentationlearningcontinuoustime}. In this work, we focus on CTDGs due to their expressive power and alignment with real-world interaction patterns.

\subsection{Adversarial Attacks on Graphs}
Adversarial attacks fall into two broad categories:\\
\textit{Evasion Attacks:} Manipulates inputs at inference time to mislead the model after training.\\
\textit{Poisoning attacks:} Injects perturbations \textit{during training} to corrupt the learned representations.
\textbf{\ourmethod\ is a poisoning attack}, aiming to degrade overall model performance by manipulating the training graph.
\paragraph{Poisoning static graphs vs. temporal graphs:} A substantial body of literature explores poisoning attacks on static graphs~\citep{bojchevski2019adversarial, li2020adversarial}, where edge or node feature perturbations can significantly impact downstream tasks like node classification or link prediction. However, directly applying these methods to CTDGs is ineffective due to the evolving nature of temporal graphs. Specifically, past perturbations quickly lose influence as new edges arrive, and future edges are not observable at attack time, preventing attackers from crafting precise interventions. This dynamic nature limits the transferability of static attack strategies to temporal domains.

\paragraph{Poisoning DTDGs vs CTDGs:} \citet{chen2021time, sharma2023temporal} extend poisoning strategies to DTDGs by targeting discrete graph snapshots. While effective in short-term settings, these methods are inherently limited in scope. Perturbations in DTDGs only influence the snapshot in which they occur and lack temporal propagation, making them ill-suited for CTDGs where edge interactions carry real-valued timestamps and temporal continuity is essential.

\paragraph{Existing Literature on Poisoning CTDGs:} T-SPEAR~\cite{lee2024spear} is, to our knowledge, the first dedicated poisoning attack on CTDGs. It trains a surrogate model to select edges that, when inserted, disrupt downstream predictions while adhering to 4 unnoticeability constraints~(similar but weaker than ours). They also present an adversarial defense strategy against poisoning attacks called T-shield. While effective, T-SPEAR has critical limitations: (1) Assumes full access to the dataset, which is not true in real-world scenarios, as adversary often has access to just training dataset. (2) Requires significant compute to train and optimize a surrogate model. (3) Introduces edges but does not consider deleting critical ones, limiting perturbation scope. (4) The effectiveness of the attack depends on the accuracy of the surrogate model used. 

%% file: sections/sec-2.5-preliminaries.tex
\section{Preliminaries}
\paragraph{Dynamic graph setting:} We consider a CTDG defined as $G = (V, E)$, where $V$ is the set of $n$ nodes, and $E$ is the set of $m$ timestamped edges. Each edge is a tuple $(u, v, t)$, where $u, v \in V$ denote the interacting nodes and $t \in \mathbb{R}$ is the time of interaction. Multiple edges can exist between the same node pair at different times, capturing repeated interactions over time.

\paragraph{Temporal PageRank:} Temporal PageRank (TPR) extends classical PageRank to temporal graphs by replacing static walks with time-respecting walks~\cite{rozenshtein2016temporal}. Let $Z^T(v, u \mid t)$ denote the set of temporal walks~(cf Appendix \ref{glossary}) from node $v$ to $u$ that occur strictly before time $t$. The probability of a walk $z$ is defined as:

\[
\Pr'[z \in Z^T(v, u \mid t)] = \frac{c(z \mid t)}{\sum\limits_{z' \in Z^T(v, x \mid t), \, x \in V, \, |z'| = |z|} c(z' \mid t)},
\]

where $c(z \mid t)$ is the decay-weighted count of $z$:

\[
c(z \mid t) = (1 - \beta)\!
\mathop{\prod_{\substack{((u_{i-1},u_i,t_i),\\(u_i,u_{i+1},t_{i+1}))\in z}}}
\beta^{\,\left|\left\{(u_i,y,t')\mid t'\in[t_i,t_{i+1}],\,y\in V\right\}\right|}.
\]

Transitions are penalized exponentially by the number of intervening interactions to model temporal decay. The final TPR score of node $u$ at time $t$ is:

\[
r(u, t) = \sum_{v \in V} \sum_{k=0}^{t} (1 - \alpha) \alpha^k \sum_{\substack{z \in Z_T(v, u \mid t) \\ |z| = k}} \Pr'[z \mid t],
\]

where $\alpha$ is the jump probability. This formulation biases the walk toward shorter, temporally coherent paths and naturally captures influence drift in evolving graphs.

\section{Problem Formulation}
\subsubsection{Attacker's objective.}
Given a clean CTDG $G$, the attacker seeks to degrade the performance of a Temporal Graph Neural Network (TGNN) on dynamic link prediction by injecting adversarial edges and/or removing crucial ones. The perturbed graph is denoted as:
\[
\tilde{G} = (V, (E \setminus E') \cup \tilde{E}),
\]
where $E' \subset E$ are the removed edges, $\tilde{E}$ are the adversarial (inserted) edges, and the total number of modifications is bounded by a fixed budget $\Delta$. This is a \textbf{poisoning attack}, the graph is modified before training, and the TGNN is trained on the corrupted $\tilde{G}$.

\subsubsection{Attacker's knowledge.} \label{sec:adversarialAttackonDG}
We assume a strict black-box setting: the attacker has no knowledge of the TGNN architecture, loss function, or gradients. However, the attacker can observe the training portion of the dataset, a common assumption in graph poisoning ~\cite{lee2024spear, ma2020practical, Z_gner_2018}. Unlike previous work like T-spear~\cite{lee2024spear}, we do not assume access to validation/test splits.

\subsubsection{Unnoticeability Constraints}
\label{sec:unnoticeability_constraints}
To ensure the attack remains stealthy and realistic, we adopt and strengthen the four unnoticeability constraints introduced by T-spear:

\begin{itemize}
    \item \textbf{(C1) Perturbation Budget:} The total number of modified edges (inserted + removed) is bounded by $\Delta = \lfloor p \cdot |E| \rfloor$, where $p$ is the perturbation rate.

    \item \textbf{(C2) Temporal Feasibility:} Timestamps of inserted edges are drawn from the same distribution as original timestamps, i.e., $\tilde{t} \sim P_t(E)$.

    \item \textbf{(C3) Node Activity Window:} Inserted edges may only connect nodes that have been active within a temporal window $W$ around the chosen timestamp $\tilde{t}$. That is, both endpoints of $\tilde{e} = (u, v, \tilde{t})$ must appear in $V_{i, W} = \{u_j, v_j \mid j \in [i-W+1, i]\}$.

    \item \textbf{(C4) Degree Preservation:} We impose a stronger variant of node-level stealth by matching the degree distribution of each node before and after perturbation. This prevents attackers from noticeably increasing or decreasing a node’s interaction footprint.
\end{itemize}

%% file: sections/sec-3-Method.tex
\section{Methodology}
We propose a two-stage adversarial attack framework, \textbf{Sparsification} followed by \textbf{Replacement}, designed to degrade TGNNs trained on CTDGs. In this section, we explain \ourmethod~in detail.



\subsection{Step 1: Sparsification}
\label{sec:sparsification}
We perform the Sparsification of a CTDG in 16 different ways. We classify them broadly into broad types:
\paragraph{Edge Sparsification Strategy:} For each timestamp $t_i$, we construct a static aggregated graph $G^{(i)} = (V, E^{(i)})$ where $E^{(i)} = \{(u,v,t) \in E : t \leq t_i\}$. We define a general temporal importance function $H: V \times V \times \mathbb{R} \rightarrow \mathbb{R}$ that computes the importance score for any edge $(u,v,t_j)$ as:
\begin{equation}
H(u,v,t_j) = f((u,v), G^{(j)}),
\end{equation}
where $f$ represents any graph-theoretic heuristic applied to edge $(u,v)$ within the context of static aggregation $G^{(j)}$. Edges are ranked by $H(u,v,t_j)$ in descending order, and the top-$\Delta$ highest scoring edges are selected for removal to achieve the desired $p$. We consider 4 heuristic strategies and also a Random edge removal strategy~(cf Appendix~\ref{A:heuristics}).

\paragraph{Timestamp Sparsification Strategy:} We identify critical timestamps by measuring temporal drift in node importance through generalized distance metrics. Let $r^{(t_i)} \in \mathbb{R}^{|V|}$ denote the TPR vector at timestamp $t_i$. We define temporal drift at timestamp $t_i$ as:
\begin{equation}
\delta^{(t_i)} = d(r^{(t_i)}, r^{(t_{i-1})}),
\end{equation}
where $d(\cdot, \cdot)$ represents a generalized distance function measuring importance shift between consecutive timestamps. High $\delta^{(t_i)}$ values indicate volatile periods where minimal perturbations significantly impact learned temporal dynamics. We rank all timestamps by their drift values and select the top timestamps with highest $\delta^{(t_i)}$ scores. If the final timestamp exceeds $\Delta$, we select exactly the remaining number of edges at that timestamp to meet the budget. Similar to edge sparsification, where heuristic scores can be replaced with various graph-based metrics, temporal drift computation supports multiple distance metrics (eg. $\ell_2$-norms, cosine distance, Jaccard distance etc). We consider with 11 different distance metrics~(cf Appendix \ref{A:heuristics}).

\subsection{Step 2: Adversarial Negative Sampling}
\label{sec:negativeSampling}
After removing critical edges in Step 1, we insert feasible negative edges to maintain graph density and temporal dynamics. This is designed to respect unnoticeability constraints (C2–C4) without requiring access to model gradients or internals. Our negative sampling algorithm operates as follows: for each removed edge, we insert a new edge that (i) mimics natural temporal patterns, (ii) connects temporally active nodes, and (iii) preserves the degree distribution. 

\textbf{C2:} We first sample candidate timestamps from the empirical distribution of existing edge times using kernel density estimation (KDE). This ensures inserted edges follow the same temporal rhythm as genuine interactions, satisfying temporal stealth (\textbf{C2}).

\textbf{C3:} Given a sampled timestamp, we construct a candidate node pool by selecting nodes active within a local time window $W$ around that timestamp. This enforces the activity constraint (\textbf{C3}), ensuring inserted edges do not connect inactive or dormant nodes, which would appear suspicious under standard monitoring tools. We adopt the same $W$ used by T-spear~\cite{lee2024spear}. For bipartite datasets, endpoints are sampled from disjoint node sets. Additionally, we enforce that node pairs $(u, v)$ are new to the graph—no prior interaction exists in either direction—preventing semantic leakage and contradictions with observed history.

\textbf{C4:} To maintain degree consistency (\textbf{C4}), we track deletions and insertions per node and update candidate pools accordingly. New edges are selected to keep each node's in-degree and out-degree statistically unchanged, guarding against structural anomalies such as hub overloading or connectivity spikes.

Our constraint-aware algorithm iteratively samples valid edge–timestamp pairs that satisfy all stealth constraints. When no valid candidates remain due to constraint overlap or capacity limits, we trigger a recovery step that reinitializes the KDE over timestamps and samples a fresh set of candidate times. This resampling enables the algorithm to explore new feasible timestamps without relaxing any constraints, ensuring the perturbation budget is met while preserving unnoticeability.

These steps form a principled, constraint-driven negative sampling routine. Unlike learned adversarial strategies requiring surrogate models, our method generates structurally and temporally plausible perturbations without supervision, yielding effective attacks in low-resource settings. Alternative approaches like random sampling (selecting edges arbitrarily out of temporal order) or Havel-Hakimi~\cite{havelhakimi} are insufficient—random sampling lacks strategic impact, while Havel-Hakimi preserves degrees without addressing temporal dynamics or unnoticeability requirements.
We provide the pseudocode for \ourmethod's negative sampling algorithm in Appendix \ref{sec:NS_algo}.

\subsection{Complexity Analysis}
We analyze the time and space complexity of each component in \ourmethod. 
\paragraph{Time complexity of Temporal PageRank}
\textit{The time complexity of the TPR algorithm is $O(|E| + |V|)$, where $|E|$ is the number of edges in the input graph and $|V|$ is the number of vertices.}

\paragraph{Space Complexity of Temporal PageRank}\label{lem:spaceTPR}
\textit{Let $|V|$ be the number of distinct nodes and $|T|$ the number of distinct
timestamps in the input sequence.
The combined data structures of TPR and
TER occupy}
\[
  O\!\bigl(|V|\,|T|\bigr)\text{\;memory.}
\]

\paragraph{Time Complexity of TimeStamp Selector}
\textit{The time complexity of our method is }
\[
O\big(|V|(\log|V| + d_{\text{max}}k\log k) + |E|\log|E|\big),
\]
\textit{where $|V|$ the number of vertices $|E|$ is the number of edges, $d_{\text{max}}$ is the maximum degree of any vertex, and $k$ is the average number of timestamps per edge.}

\paragraph{Space complexity of Timestamp Selector}\label{lem:spaceTSS}
\textit{Let $W$ be the length of the time window used when gathering
candidate edges.
Algorithm~{Timestamp Selector} uses}
\[
  O\!\bigl(|V| \;+\; |E|\,W\bigr)\text{\;\textit{memory}.}
\]
Detailed proofs and empirical analysis are provided in Appendix~\ref{app:complexity}, where we also demonstrate that \textbf{\ourmethod~achieves up to 10× speedup over T-spear}.

%% file: sections/sec-4-Experiments.tex
\section{Experiment Setup}
We provide details about datasets and models along with metrics and baselines used in experiments to show the effectiveness of our method (cf. Appendix \ref{dataset-details}).\\
\textbf{Datasets:} We perform experiments on 4 benchmark real-world datasets: Wikipedia \cite{kumar2019predicting}, MOOC \cite{feng2019understanding}, UCI \cite{panzarasa2009patterns}, and Enron \cite{shetty2004enron}. For all datasets, we adopt a consistent chronological split of 70\%-15\%-15\% for training-validation-test sets, following the methodology in \cite{rossi2020temporal}. Detailed descriptions of these datasets are provided in Appendix \ref{dataset-details}.

\textbf{Models:} To evaluate the effectiveness of \ourmethod, we conduct experiments on 4 SotA TGNNs: TGN \cite{rossi2020temporal}, JODIE \cite{kumar2019predicting}, TGAT \cite{xu2020inductive}, and DySAT \cite{sankar2018dynamic}. 

\textbf{Metrics:} To ensure consistency and comparability, we utilize Mean Reciprocal Rank (MRR), which is a robust and a widely recognized evaluation metric for ranking tasks. 

\textbf{Baselines:} We evaluate our method against 11 baselines: T-spear (SotA CTDG poisoning attack), 5 edge addition baselines, i.e., ADD, which were used in the T-spear paper, and 5 edge removal baselines, i.e., REM, extrapolated from static graph methods. Unlike the original T-spear work, which poisons the entire dataset, including the validation and test sets, \textbf{we poison only the training set even for T-spear to ensure a fair comparison with our method, which accounts for any differences in reported T-spear performance numbers}. Detailed explanation of baselines is in Appendix \ref{baseline-details}.



\begin{figure*}[!t]
\centering
\begin{subfigure}{0.48\linewidth}
    \centering
    \includegraphics[width=\linewidth]{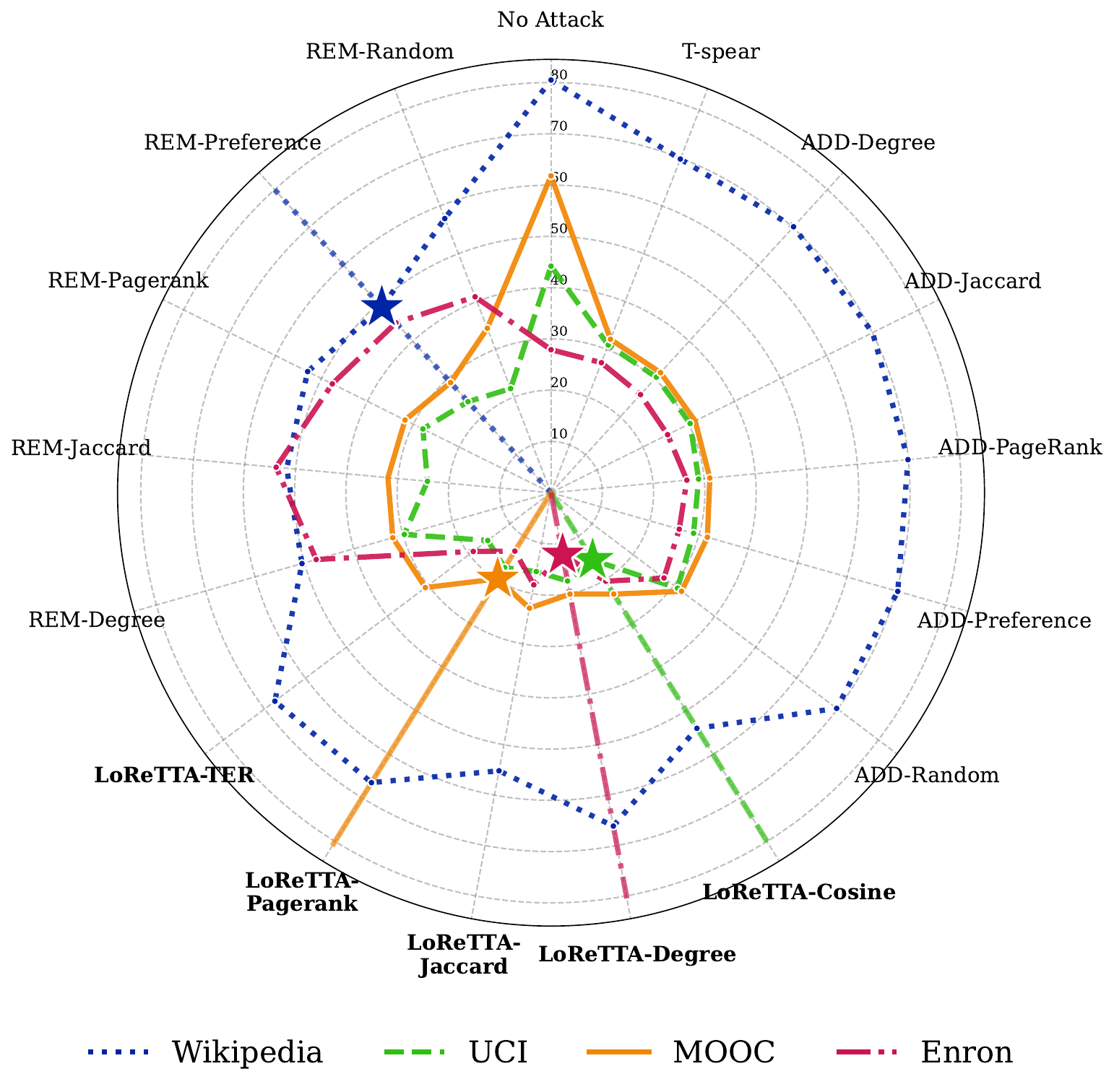}
    \caption{Attack Performance: \textbf{TGN}}
    \label{fig:tgn_radar}
\end{subfigure}
\begin{subfigure}{0.48\linewidth}
    \centering
    \includegraphics[width=\linewidth]{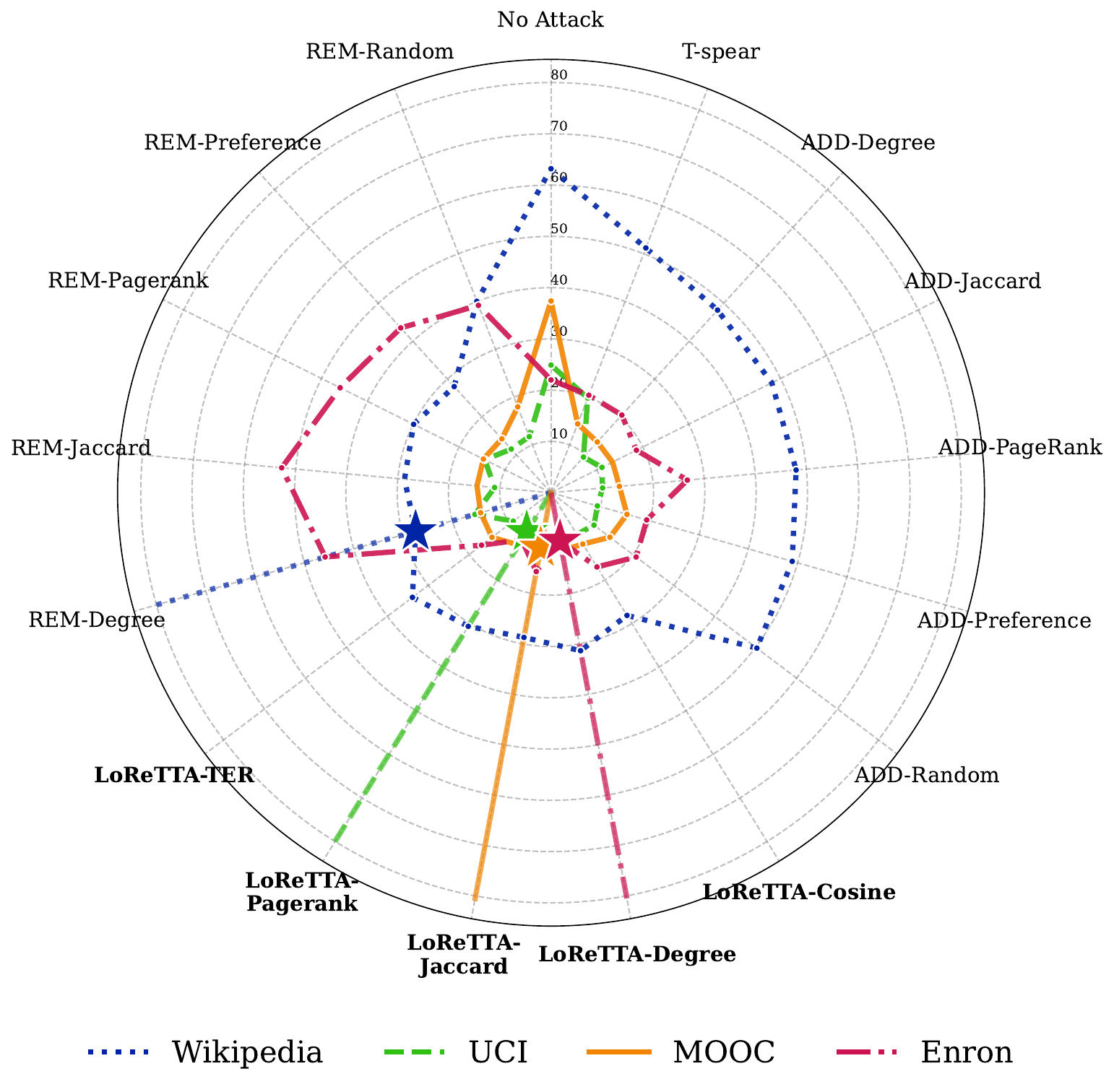}
    \caption{Attack Performance: \textbf{JODIE}}
    \label{fig:jodie_radar}
\end{subfigure}
\begin{subfigure}{0.48\linewidth}
    \centering
    \includegraphics[width=\linewidth]{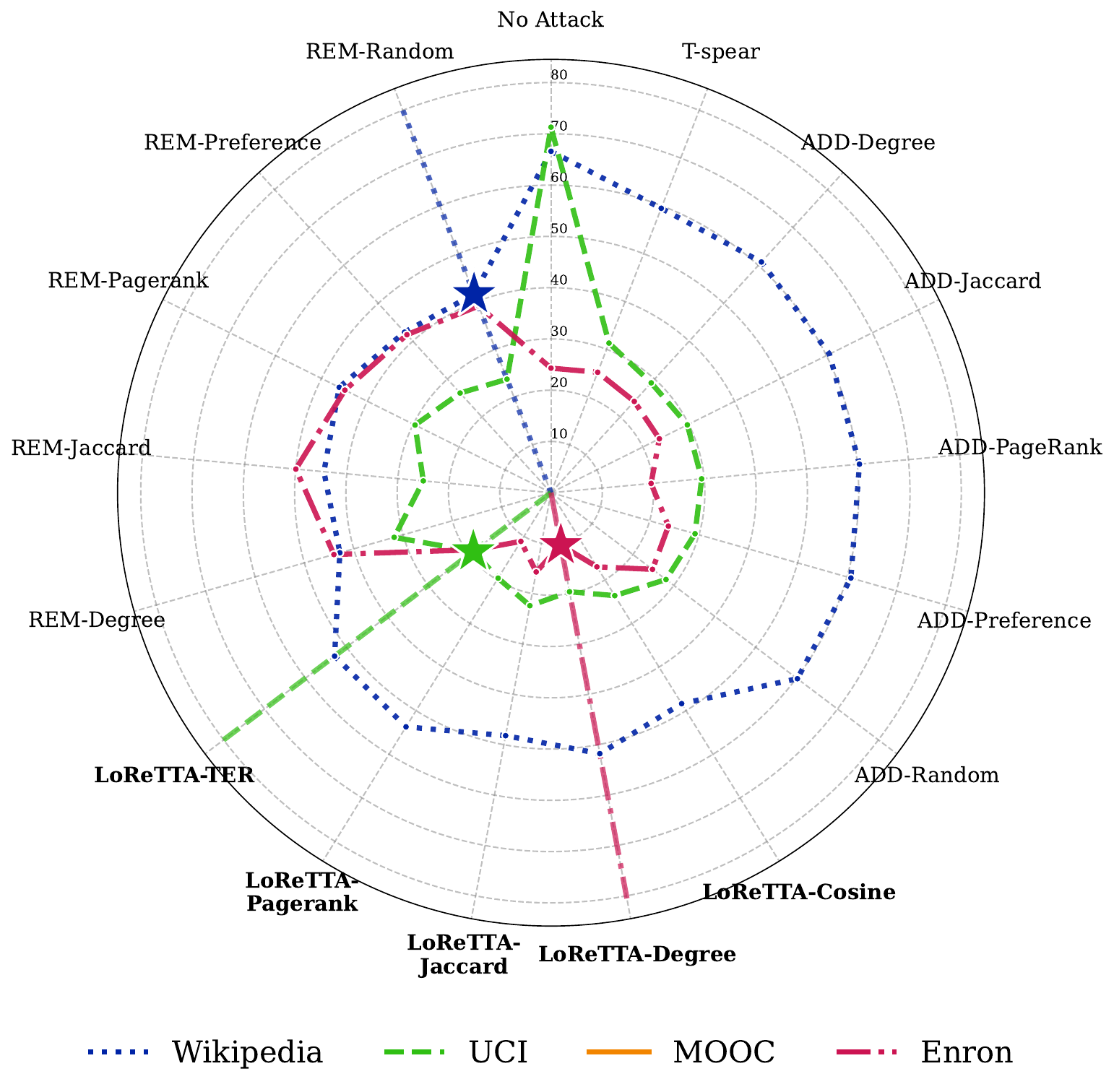}
    \caption{Attack Performance: \textbf{DySAT}}
    \label{fig:dysat_radar}
\end{subfigure}
\begin{subfigure}{0.48\linewidth}
    \centering
    \includegraphics[width=\linewidth]{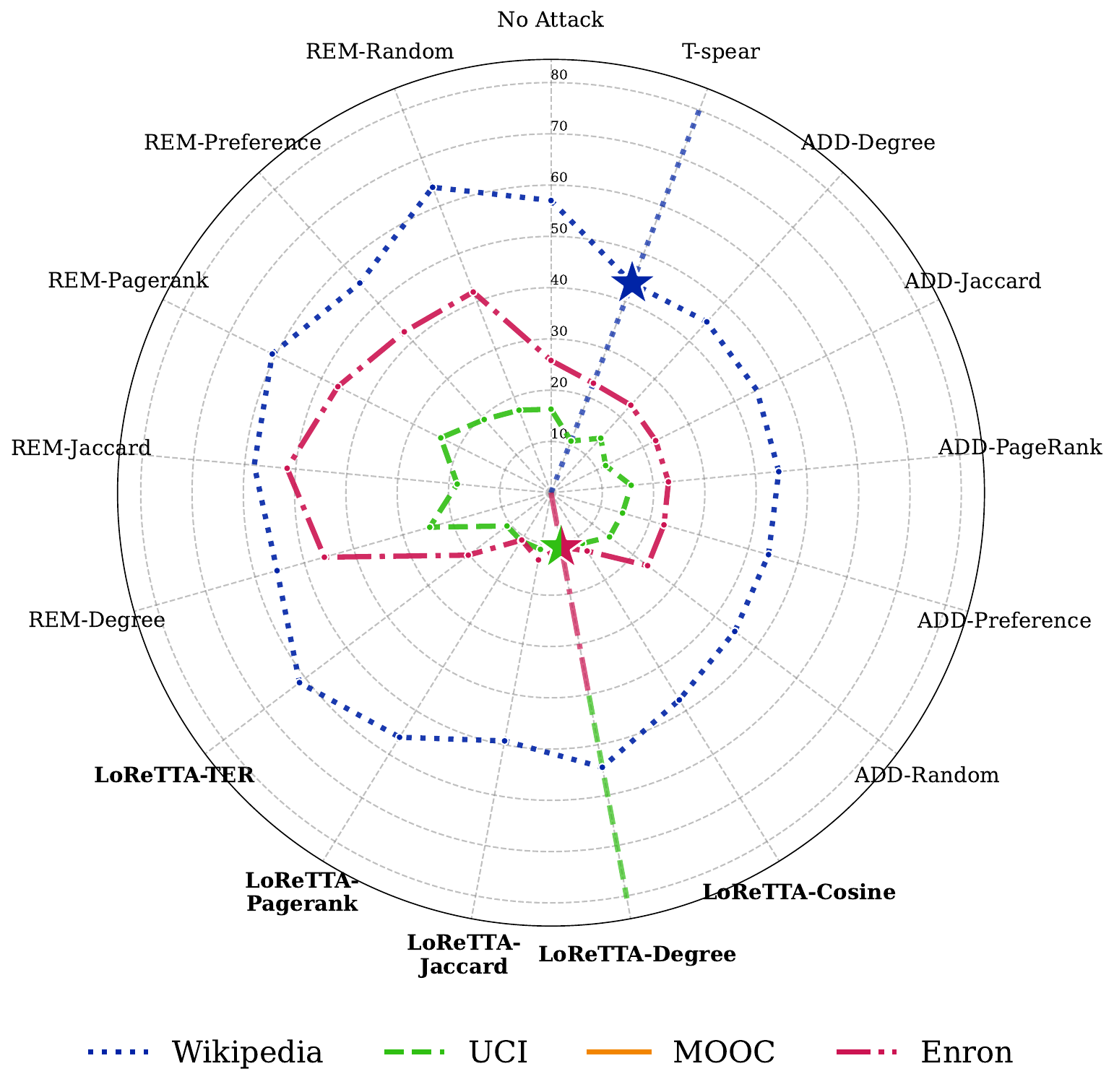}
    \caption{Attack Performance: \textbf{TGAT}}
    \label{fig:tgat_radar}
\end{subfigure}
\caption{
    Performance of our attack on each model across datasets. Lower values indicate stronger attack impact. Clean baselines are shown under No-Attack. DySAT and TGAT fail on MOOC due to Out-Of-Memory error. Naive Jaccard is omitted for bipartite graphs~(MOOC and UCI), where it is not applicable. Dashed radial lines highlight the most effective attack for each dataset-model pair. \ourmethod-Degree is the best metric for both Enron and UCI on TGAT.
}
\label{fig:attack_perf_radar}
\end{figure*}

%% file: sections/sec-5-Discussion.tex
\section{Discussion}
In this section, we evaluate performance of \ourmethod~against baselines~(cf.\S\ref{sec:attack_perf}), highlight its robustness to SotA defenses~(cf. \S\ref{robustness}), anomaly detection algorithms~(cf.\S\ref{sec:anomaly_detection}) and present our ablation studies(cf. \S\ref{sec:ablation_studies}) 
\subsection{Attack Performance}
\label{sec:attack_perf}
\begin{table*}[htpb]
\centering
\small
\setlength{\tabcolsep}{6pt}
\begin{tabular}{|l|cc|cc|cc|cc|}
\hline
\multirow{2}{*}{\textbf{Attack}} &
\multicolumn{2}{c|}{\textbf{SVD}} &
\multicolumn{2}{c|}{\textbf{Cosine}} &
\multicolumn{2}{c|}{\textbf{T-shield-F}} &
\multicolumn{2}{c|}{\textbf{T-shield}} \\ \cline{2-9}
& \textbf{Wikipedia} & \textbf{UCI}
& \textbf{Wikipedia} & \textbf{UCI}
& \textbf{Wikipedia} & \textbf{UCI}
& \textbf{Wikipedia} & \textbf{UCI} \\ \hline
Degree       & $40.96\!\pm\!1.29$ & $34.04\!\pm\!1.78$ & $27.57\!\pm\!1.05$ & $ 8.05\!\pm\!0.37$ & $10.45\!\pm\!0.65$ & $ 5.78\!\pm\!0.62$ & $10.30\!\pm\!0.90$ & $ 5.98\!\pm\!0.32$ \\
PageRank     & $40.27\!\pm\!0.88$ & $34.44\!\pm\!1.26$ & $27.55\!\pm\!1.40$ & $ 7.72\!\pm\!0.38$ & $10.98\!\pm\!0.74$ & $ 5.86\!\pm\!0.40$ & $11.03\!\pm\!0.80$ & $ 5.62\!\pm\!0.41$ \\
Cosine       & $46.01\!\pm\!1.36$ & $34.34\!\pm\!1.26$ & $17.87\!\pm\!2.11$ & $ 7.63\!\pm\!0.90$ & $10.32\!\pm\!0.57$ & $ 6.25\!\pm\!0.41$ & $ 9.62\!\pm\!1.08$ & $ 6.10\!\pm\!0.79$ \\
Jaccard      & $48.94\!\pm\!0.97$ & $33.70\!\pm\!1.69$ & $19.33\!\pm\!2.19$ & $ 7.74\!\pm\!0.57$ & $ 9.08\!\pm\!0.74$ & $ 5.69\!\pm\!0.52$ & $ 9.61\!\pm\!0.37$ & $ 5.50\!\pm\!0.27$ \\
TER          & $47.99\!\pm\!1.39$ & $34.46\!\pm\!0.94$ & $32.11\!\pm\!0.97$ & $ 7.40\!\pm\!0.76$ & $11.31\!\pm\!1.07$ & $ 6.20\!\pm\!0.57$ & $11.38\!\pm\!1.05$ & $ 5.96\!\pm\!0.41$ \\ \hline
Clean        & $80.50 \pm  0.50$  & $44.2 \pm 0.4$           & $80.50 \pm  0.50$           & $44.2 \pm 0.4$             & $80.50 \pm  0.50$            &  $44.2 \pm 0.4$           & $80.50 \pm  0.50$            & $44.2 \pm 0.4$                      \\ \hline
\end{tabular}
\caption{Performance~(\% MRR) of variants of \ourmethod~against 4 SotA defenses. The "Clean" row establishes the baseline performance of TGN without any perturbations or defenses applied.}
\label{tab:robustness_results}
\end{table*}
We perform a comprehensive evaluation of our proposed attack method against 11 baseline approaches~(cf Appendix \ref{baseline-details}). Figure~\ref{fig:attack_perf_radar} presents the average performance of 5 runs of all victim models under various attack methods at a perturbation rate of $0.3$. The results for MOOC on TGAT and DySAT are not included due to out-of-memory errors caused by the size of the models.

\ourmethod~consistently beats all the baselines in the majority of cases, demonstrating superior attack effectiveness. The single exception occurs on the Wikipedia dataset, where REM baselines achieve comparable or slightly better performance. This anomaly stems from Wikipedia's unique structural properties that make sparsification-only attacks particularly effective. Specifically, Wikipedia exhibits highly modular semantic clusters where users contribute predominantly to domain-specific areas. Sparsification attacks exploit this modularity by removing high-signal edges within semantic clusters, directly undermining the model's ability to capture critical local coherence patterns. In contrast, our method's negative sampling component introduces cross-domain edges that inadvertently provide implicit regularization, encouraging broader generalization across otherwise disjoint clusters and partially counteracting the sparsification damage.
However, it should be noted that REM baselines achieve this performance by adhering only to the C1 constraint, thereby disrupting CTDG structure more extensively than our approach, which respects both C1 and C2 constraints. We provide a detailed analysis of this phenomenon in Appendix \ref{sec:wiki_sparse_effectiveness}.

Importantly, \ourmethod~demonstrates clear superiority across all other datasets, proving our method's general effectiveness. Furthermore, \ourmethod~doesn't just work for a specific sparsification heuristic. We consistently outperform many baselines for all 5 chosen sparsification strategies, showing the stability of our attack framework and its independence from specific sparsification techniques. The attack performance in tabular form is in Appendix \ref{A:attack_table}. We also assess dominance using a one-sided exact binomial sign test across baselines; see Appendix~\ref{sec:sign-test} and Table~\ref{tab:loretta_sign_all_pairs}.

\subsection{Robustness against SotA Defenses}\label{robustness}
We evaluate the robustness of our proposed \ourmethod~against 4 SotA adversarial defense methods using TGN on Wikipedia and UCI datasets with $p$ as 0.3, as presented in Table~\ref{tab:robustness_results}. The defenses include 2 static graph methods (SVD-based reconstruction and cosine similarity filtering) and 2 variations of Tshield~\cite{lee2024spear} (Tshield-F and Tshield), which represent the current SotA CTDG defense methods. 

Our results demonstrate that even after applying these defense mechanisms, victim models consistently fail to recover their original clean performance levels. Notably, we observe that Cosine similarity filtering, T-shield, and T-shield-F defenses cause additional performance degradation beyond that induced by \ourmethod~alone. This counterintuitive result occurs because these filtering-based approaches attempt to identify and remove adversarial edges from the dataset. However, when adversarial edges are incorrectly identified, true edges are inadvertently removed, resulting in further performance deterioration.
We experimentally validate this hypothesis in Appendix \ref{A:rob_hypothesis_ablation}, where we discovered that only approximately 30\% of the filtered edges correspond to actual adversarial modifications. This finding not only validates the unnoticeability of \ourmethod~but also underscores its robustness against current defense strategies, highlighting the urgent need for more robust defense methods

\subsection{Robustness against Anomaly Detection methods}
\begin{table}[H]
\centering
\small
\setlength{\tabcolsep}{3pt} 
\renewcommand{\arraystretch}{0.9}
\begin{tabular}{@{}lcccccccc@{}}
\toprule
& \multicolumn{8}{c}{\textbf{Anomaly Detection Methods}} \\
\cmidrule(lr){2-9}
\multirow{2}{*}{\textbf{Attack}} & \multicolumn{2}{c}{\textbf{MIDAS}} & \multicolumn{2}{c}{\textbf{F-FADE}} & \multicolumn{2}{c}{\textbf{AnoEdge-L}} & \multicolumn{2}{c}{\textbf{AnoEdge-G}} \\
\cmidrule(lr){2-3} \cmidrule(lr){4-5} \cmidrule(lr){6-7} \cmidrule(lr){8-9}
& \textbf{P} & \textbf{R} & \textbf{P} & \textbf{R} & \textbf{P} & \textbf{R} & \textbf{P} & \textbf{R} \\
\midrule
\multicolumn{9}{@{}l}{\textbf{Wikipedia Dataset}} \\
\midrule
Degree   & 0.35 & 0.52 & 1.00 & 0.28 & 0.47 & 0.45 & 0.39 & 0.55 \\
PageRank & 0.35 & 0.53 & 1.00 & 0.30 & 0.49 & 0.46 & 0.40 & 0.55 \\
Cosine   & 0.26 & 0.70 & 1.00 & 0.21 & 0.50 & 0.38 & 0.46 & 0.74 \\
Jaccard  & 0.28 & 0.67 & 1.00 & 0.18 & 0.50 & 0.36 & 0.42 & 0.75 \\
TER      & 0.37 & 0.73 & 1.00 & 0.15 & 0.46 & 0.31 & 0.29 & 0.73 \\
\midrule
\multicolumn{9}{@{}l}{\textbf{UCI Dataset}} \\
\midrule
Degree   & 0.48 & 0.29 & 0.98 & 0.44 & 0.28 & 0.26 & 0.32 & 0.44 \\
PageRank & 0.51 & 0.33 & 0.98 & 0.45 & 0.33 & 0.26 & 0.34 & 0.48 \\
Cosine   & 0.51 & 0.34 & 0.33 & 0.67 & 0.46 & 0.51 & 0.59 & 0.82 \\
Jaccard  & 0.57 & 0.36 & 0.28 & 0.84 & 0.49 & 0.45 & 0.65 & 0.84 \\
TER      & 0.63 & 0.62 & 0.26 & 1.00 & 0.52 & 0.37 & 0.44 & 0.67 \\
\bottomrule
\end{tabular}
\caption{Anomaly–class precision (P) and recall (R) of anomaly detection methods after perturbed using \ourmethod}
\label{tab:anomaly_table_single_col}
\end{table}
\label{sec:anomaly_detection}
When adversaries employ data poisoning techniques to degrade model performance, anomaly detection systems can be deployed to identify adversarial edges and effectively clean the dataset before training. Consequently, it is critical for poisoning attacks to remain undetectable to such systems. To demonstrate \ourmethod's stealth capabilities, we evaluate its robustness against 4 SotA edge-stream anomaly detection methods: MIDAS~\cite{midas}, F-FADE~\cite{f-fade}, AnoEdge-L, and AnoEdge-G~\cite{anoedge}. We focus on unsupervised approaches due to their practical relevance, as supervised methods require labeled anomaly datasets that are typically unavailable in real-world deployment scenarios.

\begin{table*}[htbp]
\centering
\resizebox{\textwidth}{!}{%
\begin{tabular}{|l|cccc|cccc|}
\hline
\textbf{Method} & 
\multicolumn{4}{c|}{\textbf{Wikipedia}} & 
\multicolumn{4}{c|}{\textbf{UCI}} \\ \hline

\textbf{Attacker's Knowledge} & 
\textbf{0.2} & \textbf{0.4} & \textbf{0.6} & \textbf{0.8} &
\textbf{0.2} & \textbf{0.4} & \textbf{0.6} & \textbf{0.8} \\ \hline

{Cosine} & 57.16 ± 0.77 & 58.90 ± 0.85 & 58.00 ± 0.91 & 58.25 ± 1.25 & 33.58 ± 1.58 & 35.84 ± 1.68 & 33.42 ± 1.78 & 34.27 ± 2.40 \\
{Random}                 & 62.61 ± 0.81 & 63.17 ± 1.26 & 62.61 ± 1.13 & 62.61 ± 0.47 & 35.06 ± 1.45 & 34.83 ± 2.00 & 34.60 ± 2.52 & 35.32 ± 1.34 \\
{Jaccard} & 58.56 ± 0.82 & 58.45 ± 0.79 & 58.60 ± 0.51 & 59.97 ± 1.47 & 35.54 ± 1.66 & 36.39 ± 1.68 & 32.69 ± 1.53 & 33.95 ± 1.76 \\
{Pagerank}              & 61.18 ± 0.69 & 64.43 ± 0.42 & 64.91 ± 1.33 & 67.40 ± 0.89 & 35.01 ± 2.32 & 35.83 ± 1.87 & 32.33 ± 2.62 & 36.10 ± 0.65 \\
{Degree}                & 60.78 ± 0.70 & 64.40 ± 0.53 & 65.11 ± 0.31 & 67.06 ± 1.01 & 34.64 ± 2.67 & 35.11 ± 1.85 & 32.21 ± 1.86 & 35.38 ± 1.45 \\
{TER}  & 60.78 ± 0.63 & 63.04 ± 1.04 & 65.93 ± 0.51 & 66.78 ± 1.48 & 34.74 ± 2.54 & 33.57 ± 2.47 & 32.53 ± 1.83 & 36.36 ± 0.86 \\
\hline
\end{tabular}%
}
\caption{Performance (\%) of different removal strategies under varying levels of adversary's knowledge on the {Wikipedia} and {UCI} datasets. Each value is reported as mean ± standard deviation over 5 runs.}
\label{tab:adversary_knowledge}
\end{table*}

Table~\ref{tab:anomaly_table_single_col} presents the precision and recall scores of anomaly detection methods on the adversarial edges class when \ourmethod~perturbs 30\% of the training data for the TGN model across Wikipedia and UCI datasets. Each anomaly detection system assigns a score to every edge, where higher scores indicate greater likelihood of being anomalous. We employ Youden's thresholding method~\cite{youden} to determine the optimal threshold for anomaly classification, as it minimizes both false positives and false negatives.

The results reveal a fundamental trade-off inherent in anomaly detection systems. High precision with low recall indicates accurate identification of true anomalies but at the cost of allowing many perturbed edges to remain undetected. Conversely, high recall with low precision typically results in numerous false positives that erroneously classify legitimate edges as anomalies, potentially degrading model performance when removed, as observed in ~\S\ref{robustness}. 
Critically, across all dataset-attack-detector combinations, no method achieves both precision and recall scores exceeding 0.7 simultaneously. This demonstrates that \ourmethod~successfully evades SotA anomaly detection systems and remains practically undetectable in realistic deployment scenarios, establishing its effectiveness as a stealthy adversarial attack. We also report in Appendix \ref{A:anomaly_ablation} the AUPRC score (area under the precision–recall curve) for each anomaly detector, i.e., a threshold-free summary that integrates performance across all decision thresholds.

\subsection{Ablation Studies}\label{sec:ablation_studies}
In this section, we analyze the sensitivity of \ourmethod~to $p$ and the knowledge of the training dataset.
\paragraph{Impact of the Knowledge of Adversary:}
Table~\ref{tab:adversary_knowledge} demonstrates the effect of increasing knowledge and access to the training dataset to the adversary on TGN performance degradation using \ourmethod~on Wikipedia. We discover that increasing adversarial knowledge does not always lead to more effective attacks. Heuristic strategies like pagerank and degree underperform—even degrade—with more knowledge. These methods remove high-centrality nodes assuming they're critical, but often eliminate redundant or noisy nodes, unintentionally improving performance. This suggests that vulnerabilities are concentrated in a small set of influential nodes; once these are removed, additional knowledge yields diminishing or negative returns. 

Interestingly, random removal shows stable performance across knowledge levels, suggesting non-monotonicity in attack impact. The system’s response is highly non-linear; only specific perturbations degrade performance. These findings indicate that attack effectiveness depends on targeting specific vulnerable components rather than simply having more information about the system. For the adversary knowledge ablation, we perturb up to 30\% of the edges to study the effect of strong adversarial influence across strategies. We experimentally validate our hypothesis in Appendix \ref{A:ablation_knowledge_appendix}.
\paragraph{Impact of Perturbation Rate:}
Figure \ref{fig:inc_attack} demonstrates the effect of increasing $p$ on TGN performance degradation using \ourmethod~on Wikipedia, compared against T-spear. Performance degrades steadily with increasing $p$ before plateauing, indicating that \ourmethod's deterministic sparsification algorithm targets influential edges first. Beyond a certain threshold, removing additional edges yields diminishing returns as the algorithm begins targeting less critical connections that contribute minimally to model learning. Notably, \ourmethod~consistently outperforms the SotA baseline across all $p$, demonstrating superior effectiveness despite operating as a surrogate-free attack.
\begin{figure}[H]
    \centering
    \includegraphics[width=\linewidth]{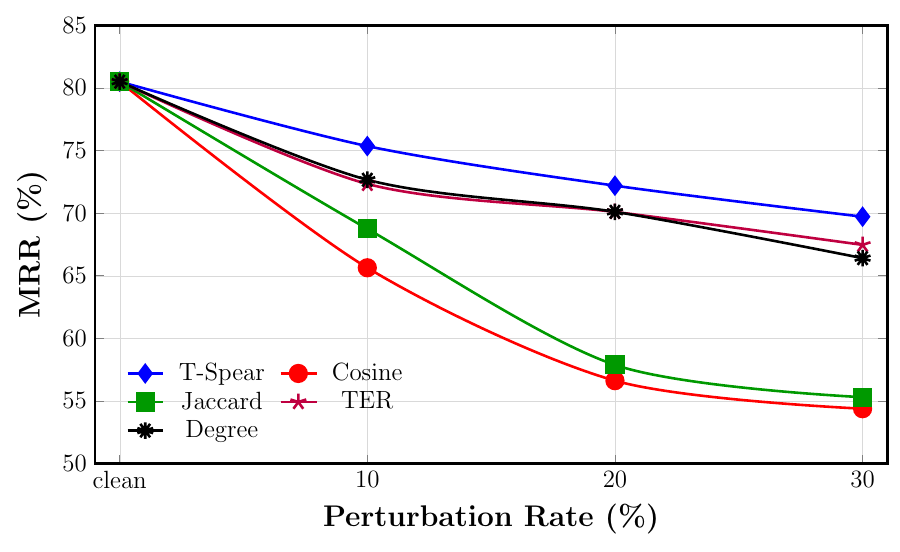}
    \caption{Effect of perturbation rate $p$ on the performance of~\ourmethod~using different sparsification strategies.}
    \label{fig:inc_attack}
\end{figure}

%% file: sections/sec-6-Conclusion.tex
\section{Conclusion}
In this paper, we propose~\ourmethod, a novel attack strategy for TGNNs that exploits temporally significant edges to achieve effective perturbations. Our method consistently outperforms baselines, including T-SPEAR, across various datasets and models, highlighting the critical role of temporal dynamics in graph vulnerabilities. \ourmethod~is also robust to SotA adversarial defenses and anomaly detection systems.

%% file: sections/section-8-appendix.tex
\newpage
\appendix

\section{Appendix}\label{Appendix}

\subsection{Detailed Explanation of Heuristics Used for Sparsification}
\label{A:heuristics}

We use deterministic sparsification strategies to remove a fixed fraction of edges from the temporal interaction graph. Each method scores edges or timestamps using a heuristic, then removes the top-ranked ones. The goal is to subtly weaken the model by pruning structurally or semantically important regions, all while ensuring reproducibility through sorted and tie-broken decisions.

Heuristics like Degree, Jaccard, PageRank, Preference, and Random (cf. \S\ref{baseline-details}) follow standard static scoring schemes. Here, we focus on sparsification methods built on two temporal graph signals: Temporal PageRank (TPR) and Temporal EdgeRank (TER).

\paragraph{Temporal PageRank Drift Heuristics.}
We track how node importance evolves over time by computing TPR scores at each timestamp. These scores are stored as a time-series vector for each node, forming a temporal embedding. To identify unstable or high-impact regions, we measure how much these vectors shift between timestamps using distance or divergence metrics.

\begin{itemize}
    \item \textbf{Mean Shift in TPR ($\text{MSS}$ \& $\text{MSS}^2$).} Both heuristics compute shifts in TPR scores across consecutive timestamps, but differ in the order of operations. In $\text{MSS}$, we first compute absolute per-node shifts and then take the mean. In $\text{MSS}^2$, we compute the mean TPR vector across nodes at each timestamp, then take the shift between these mean vectors. The former emphasizes global volatility across nodes, while the latter highlights shifts in average importance.

    \item \textbf{Cosine Distance.} For each pair of consecutive timestamps, we compute the cosine distance between global TPR vectors:
    \[
    1 - \frac{ \vec{P}^{(t)} \cdot \vec{P}^{(t+1)} }{ \|\vec{P}^{(t)}\|_2 \cdot \|\vec{P}^{(t+1)}\|_2 }
    \]
    This captures how the "direction" of node importance has changed, without being affected by absolute scale.

    \item \textbf{Jaccard Distance.} We binarize each TPR vector to retain the top-$k$ influential nodes at each timestamp, then compute:
    \[
    1 - \frac{ |B^{(t)} \cap B^{(t+1)}| }{ |B^{(t)} \cup B^{(t+1)}| }
    \]
    This captures how much the set of most important nodes changes over time.

    \item \textbf{Euclidean Distance.} This computes:
    \[
    \left\| \vec{P}^{(t)} - \vec{P}^{(t+1)} \right\|_2
    \]
    It measures the absolute magnitude of drift in node importance across timestamps.

    \item \textbf{Jensen–Shannon Divergence.} We normalize each TPR vector into a probability distribution and compute the symmetric divergence:
    \[
    \text{JSD}(Q^{(t)} \,\|\, Q^{(t+1)}) = \frac{1}{2} \text{KL}(Q^{(t)} \,\|\, M) + \frac{1}{2} \text{KL}(Q^{(t+1)} \,\|\, M)
    \]
    where $M$ is the midpoint distribution. This gives a smooth, bounded signal for representational drift.

    \item \textbf{Kullback–Leibler Divergence.} Also operating on normalized TPR vectors, this asymmetric metric measures:
    \[
    \text{KL}(Q^{(t)} \,\|\, Q^{(t+1)}) = \sum_i Q^{(t)}_i \log \frac{Q^{(t)}_i}{Q^{(t+1)}_i + \epsilon}
    \]
    It emphasizes where a node’s importance suddenly changes in one direction.

    \item \textbf{Chebyshev Distance.} This focuses on the node with the most drastic change in importance:
    \[
    \max_i |P^{(t)}_i - P^{(t+1)}_i|
    \]
    It surfaces sharp outliers in the node ranking space.

    \item \textbf{Wasserstein Distance.} Also known as Earth Mover’s Distance, this computes how much "effort" it takes to transform one TPR distribution into the next. It’s ideal for measuring both magnitude and position shifts in importance mass.
\end{itemize}

\paragraph{Temporal EdgeRank (TER) Based Heuristics.}
Unlike node-based TPR drift, these methods score and remove edges directly based on how influential they are in propagating information over time.

\begin{itemize}
    \item \textbf{TER.} For each timestamp, we compute an importance score for every edge based on its temporal propagation potential. Edges with the highest TER are removed. This method directly targets structural bottlenecks in the flow of information.

    \item \textbf{Combined TER.} This heuristic blends two views of edge influence:
    \begin{itemize}
        \item \textit{Local}: Edge importance computed only up to timestamp $t$
        \item \textit{Global}: Cumulative influence up to the final time $T$
    \end{itemize}
    These are combined into a final score to capture both short-term volatility and long-term impact. The top-scoring edges across this combined metric are removed.
\end{itemize}

All pairwise distance computations (e.g., cosine, Euclidean, JSD, Wasserstein) over Temporal PageRank vectors were performed using the \texttt{scipy.spatial.distance} and \texttt{scipy.stats} modules from the SciPy library\citep{virtanen2020scipy}.

All the above methods aim to surface graph regions where representational dynamics or structural control are most active. The TPR-based methods take a node-centric view and quantify shifts in centrality embeddings. TER-based heuristics operate at edge-level, identifying links that most affect temporal signal flow. Together, they provide diverse ways of probing graph vulnerability. 

\subsection{Attack Performance in Tabular Form}
\label{A:attack_table}
We show the Performance of 11 baselines and \ourmethod on 4 TGNN victim models on 4 datasets in Table \ref{tab:attack-results}. We show the average of 5 runs.
\begin{table*}[hbpt]
\centering
\small                     
\setlength{\tabcolsep}{4pt}
\tightTable        
\begin{tabular}{lllllll}
\toprule
\textbf{Group} & \textbf{Attack} & \textbf{Model} &
\textbf{Wikipedia} & \textbf{UCI} & \textbf{MOOC} & \textbf{Enron} \\
\midrule
\multirow{4}{*}{\parbox[c]{2.8cm}{\centering Clean}}
  & \multirow{4}{*}{No Attack}
        & TGN   & 80.5\% $\pm$ 0.5\% & 44.2\% $\pm$ 0.4\% & 61.8\% $\pm$ 2.1\% & 27.9\% $\pm$ 0.4\% \\
  &     & JODIE & 63.2\% $\pm$ 1.2\% & 24.9\% $\pm$ 0.4\% & 37.4\% $\pm$ 2.0\% & 22.0\% $\pm$ 0.4\% \\
  &     & DySAT & 66.6\% $\pm$ 0.4\% & 71.3\% $\pm$ 0.9\% & OOM                & 24.3\% $\pm$ 0.7\% \\
  &     & TGAT  & 57.0\% $\pm$ 0.6\% & 16.3\% $\pm$ 0.3\% & OOM                & 25.8\% $\pm$ 0.6\% \\
\midrule
\multirow{20}{*}{\parbox[c]{2.8cm}{\centering ADD Baselines}}
  & \multirow{4}{*}{Degree}
        & TGN   & 70.2\% $\pm$ 0.7\% & 30.5\% $\pm$ 1.5\% & 31.7\% $\pm$ 0.5\% & 25.9\% $\pm$ 0.7\% \\
  &     & JODIE & 48.2\% $\pm$ 1.8\% &  9.4\% $\pm$ 0.3\% & 13.4\% $\pm$ 0.6\% & 20.5\% $\pm$ 1.5\% \\
  &     & DySAT & 60.9\% $\pm$ 0.5\% & 29.0\% $\pm$ 1.4\% & OOM                & 24.1\% $\pm$ 0.2\% \\
  &     & TGAT  & 45.1\% $\pm$ 0.6\% & 14.4\% $\pm$ 0.5\% & OOM                & 23.1\% $\pm$ 0.2\% \\
\cline{2-7}
  & \multirow{4}{*}{Jaccard}
        & TGN   & --                & 30.3\% $\pm$ 1.6\% & --                 & 25.4\% $\pm$ 1.0\% \\
  &     & JODIE & --                & 11.1\% $\pm$ 1.8\% & --                 & 18.6\% $\pm$ 0.8\% \\
  &     & DySAT & --                & 29.7\% $\pm$ 0.8\% & --                 & 23.6\% $\pm$ 1.5\% \\
  &     & TGAT  & --                & 11.9\% $\pm$ 0.6\% & --                 & 22.8\% $\pm$ 0.8\% \\
\cline{2-7}
  & \multirow{4}{*}{PageRank}
        & TGN   & 69.9\% $\pm$ 0.2\% & 28.9\% $\pm$ 0.8\% & 31.1\% $\pm$ 0.7\% & 26.6\% $\pm$ 0.7\% \\
  &     & JODIE & 48.0\% $\pm$ 2.4\% & 10.1\% $\pm$ 1.3\% & 13.4\% $\pm$ 0.6\% & 26.7\% $\pm$ 1.0\% \\
  &     & DySAT & 60.4\% $\pm$ 1.1\% & 29.5\% $\pm$ 0.9\% & OOM                & 19.6\% $\pm$ 0.9\% \\
  &     & TGAT  & 44.6\% $\pm$ 0.4\% & 15.7\% $\pm$ 0.5\% & OOM                & 23.0\% $\pm$ 0.8\% \\
\cline{2-7}
  & \multirow{4}{*}{Preference}
        & TGN   & 70.3\% $\pm$ 1.1\% & 28.9\% $\pm$ 1.2\% & 31.7\% $\pm$ 1.2\% & 26.0\% $\pm$ 0.5\% \\
  &     & JODIE & 48.9\% $\pm$ 2.0\% &  9.4\% $\pm$ 0.3\% & 15.4\% $\pm$ 0.6\% & 19.4\% $\pm$ 0.4\% \\
  &     & DySAT & 60.8\% $\pm$ 0.8\% & 29.2\% $\pm$ 1.2\% & OOM                & 23.8\% $\pm$ 0.8\% \\
  &     & TGAT  & 44.1\% $\pm$ 0.5\% & 14.5\% $\pm$ 0.3\% & OOM                & 22.9\% $\pm$ 0.5\% \\
\cline{2-7}
  & \multirow{4}{*}{Random}
        & TGN   & 69.8\% $\pm$ 0.4\% & 30.9\% $\pm$ 1.2\% & 31.9\% $\pm$ 0.4\% & 27.6\% $\pm$ 0.7\% \\
  &     & JODIE & 50.3\% $\pm$ 2.0\% & 10.5\% $\pm$ 1.3\% & 14.4\% $\pm$ 0.5\% & 20.8\% $\pm$ 0.3\% \\
  &     & DySAT & 60.2\% $\pm$ 1.1\% & 28.1\% $\pm$ 1.9\% & OOM                & 24.8\% $\pm$ 1.1\% \\
  &     & TGAT  & 44.9\% $\pm$ 0.8\% & 14.3\% $\pm$ 0.4\% & OOM                & 23.6\% $\pm$ 0.2\% \\
\midrule
\multirow{20}{*}{\parbox[c]{2.8cm}{\centering REM Baselines}}
  & \multirow{4}{*}{Degree} & TGN   & 50.49\% $\pm$ 1.35\% & 29.74\% $\pm$ 0.56\% & 32.08\% $\pm$ 0.21\% & 47.60\% $\pm$ 0.62\% \\
  &                         & JODIE & 27.45\% $\pm$ 0.96\% & 15.43\% $\pm$ 0.49\% & 14.27\% $\pm$ 0.44\% & 45.79\% $\pm$ 0.60\% \\
  &                         & DYSAT & 42.81\% $\pm$ 1.00\% & 31.80\% $\pm$ 0.51\% & OOM                   & 44.02\% $\pm$ 0.14\% \\
  &                         & TGAT  & 55.54\% $\pm$ 0.87\% & 24.59\% $\pm$ 1.20\% & OOM                   & 45.95\% $\pm$ 0.41\% \\
\cline{2-7}
  & \multirow{4}{*}{Jaccard} & TGN   & --                    & 24.20\% $\pm$ 1.04\% & --                    & 53.83\% $\pm$ 0.82\% \\
  &                          & JODIE & --                    & 11.05\% $\pm$ 0.72\% & --                    & 52.74\% $\pm$ 0.86\% \\
  &                          & DYSAT & --                    & 25.01\% $\pm$ 0.81\% & --                    & 49.97\% $\pm$ 0.47\% \\
  &                          & TGAT  & --                    & 18.36\% $\pm$ 0.89\% & --                    & 51.68\% $\pm$ 0.20\% \\
\cline{2-7}
  & \multirow{4}{*}{Pagerank} & TGN   & 53.02\% $\pm$ 1.55\% & 27.91\% $\pm$ 1.19\% & 31.80\% $\pm$ 0.54\% & 47.64\% $\pm$ 0.33\% \\
  &                           & JODIE & 29.93\% $\pm$ 2.96\% & 14.42\% $\pm$ 0.70\% & 14.73\% $\pm$ 0.32\% & 45.94\% $\pm$ 0.61\% \\
  &                           & DYSAT & 46.13\% $\pm$ 0.64\% & 29.62\% $\pm$ 1.11\% & OOM                   & 44.90\% $\pm$ 0.65\% \\
  &                           & TGAT  & 60.72\% $\pm$ 0.58\% & 24.01\% $\pm$ 1.26\% & OOM                   & 46.43\% $\pm$ 0.41\% \\
\cline{2-7}
  & \multirow{4}{*}{Preference} & TGN   & 48.97\% $\pm$ 1.12\% & 24.03\% $\pm$ 1.71\% & 29.07\% $\pm$ 0.58\% & 44.87\% $\pm$ 0.47\% \\
  &                             & JODIE & 28.02\% $\pm$ 0.97\% & 11.52\% $\pm$ 0.77\% & 14.23\% $\pm$ 0.36\% & 43.50\% $\pm$ 0.66\% \\
  &                             & DYSAT & 42.38\% $\pm$ 1.27\% & 26.31\% $\pm$ 1.24\% & OOM                   & 41.71\% $\pm$ 0.38\% \\
  &                             & TGAT  & 55.34\% $\pm$ 1.12\% & 19.33\% $\pm$ 0.76\% & OOM                   & 42.47\% $\pm$ 0.45\% \\
\cline{2-7}
  & \multirow{4}{*}{Random} & TGN   & 57.35\% $\pm$ 1.68\% & 21.79\% $\pm$ 1.24\% & 34.40\% $\pm$ 0.17\% & 40.99\% $\pm$ 0.26\% \\
  &                         & JODIE & 40.05\% $\pm$ 1.59\% & 11.74\% $\pm$ 0.67\% & 17.96\% $\pm$ 0.34\% & 39.19\% $\pm$ 0.80\% \\
  &                         & DYSAT & 41.55\% $\pm$ 1.03\% & 23.75\% $\pm$ 0.67\% & OOM                   & 39.07\% $\pm$ 0.21\% \\
  &                         & TGAT  & 63.90\% $\pm$ 0.89\% & 17.27\% $\pm$ 0.96\% & OOM                   & 41.99\% $\pm$ 0.35\% \\
\midrule
\multirow{4}{*}{\parbox[c]{2.8cm}{\centering CTDG Poisoning Baseline}}
  & \multirow{4}{*}{T-spear}
        & TGN   & 69.8\% $\pm$ 1.1\% & 30.9\% $\pm$ 1.0\% & 32.1\% $\pm$ 0.6\% & 27.2\% $\pm$ 0.4\% \\
  &     & JODIE & 51.2\% $\pm$ 1.9\% & 19.8\% $\pm$ 0.8\% & 14.4\% $\pm$ 0.3\% & 20.4\% $\pm$ 0.7\% \\
  &     & DySAT & 59.5\% $\pm$ 0.7\% & 31.3\% $\pm$ 1.3\% & OOM                & 25.2\% $\pm$ 1.7\% \\
  &     & TGAT  & 43.8\% $\pm$ 0.9\% & 10.8\% $\pm$ 0.7\% & OOM                & 22.9\% $\pm$ 0.5\% \\
\midrule
\multirow{20}{*}{\parbox[c]{1.6cm}{\centering LoReTTA}}
  & \multirow{4}{*}{Degree} & TGN   & 66.15\% $\pm$ 0.96\% & 17.51\% $\pm$ 1.39\% & 20.11\% $\pm$ 0.40\% & 12.31\% $\pm$ 0.68\% \\
  &                         & JODIE & 31.35\% $\pm$ 1.22\% &  9.97\% $\pm$ 1.23\% & 11.46\% $\pm$ 0.35\% &  9.55\% $\pm$ 0.40\% \\
  &                         & DySAT & 51.78\% $\pm$ 0.90\% & 19.65\% $\pm$ 1.33\% & OOM                   & 10.40\% $\pm$ 0.41\% \\
  &                         & TGAT  & 54.48\% $\pm$ 0.65\% & 10.60\% $\pm$ 0.27\% & OOM                   & 10.78\% $\pm$ 0.24\%                    \\
\cline{2-7}
  & \multirow{4}{*}{Pagerank} & TGN   & 66.51\% $\pm$ 1.22\% & 17.18\% $\pm$ 0.96\% & 19.75\% $\pm$ 0.77\% & 13.40\% $\pm$ 0.47\% \\
  &                           & JODIE & 30.65\% $\pm$ 2.21\% &  9.00\% $\pm$ 0.59\% & 11.99\% $\pm$ 0.34\% & 10.82\% $\pm$ 0.23\% \\
  &                           & DYSAT & 53.71\% $\pm$ 1.35\% & 19.60\% $\pm$ 1.14\% & OOM                   & 11.16\% $\pm$ 0.34\% \\
  &                           & TGAT  & 56.10\% $\pm$ 0.80\% & 11.05\% $\pm$ 0.59\% & OOM                   & 10.85\% $\pm$ 0.50\%                   \\
\cline{2-7}
  & \multirow{4}{*}{LoReTTA -- TER} & TGN   & 67.47\% $\pm$ 0.96\% & 15.46\% $\pm$ 1.06\% & 30.70\% $\pm$ 0.50\% & 19.00\% $\pm$ 0.53\% \\
  &                                 & JODIE & 33.86\% $\pm$ 0.56\% &  9.21\% $\pm$ 1.31\% & 14.43\% $\pm$ 0.57\% & 16.97\% $\pm$ 0.46\% \\
  &                                 & DYSAT & 52.88\% $\pm$ 1.84\% & 18.98\% $\pm$ 0.71\% & OOM                   & 18.44\% $\pm$ 0.93\% \\
  &                                 & TGAT  & 61.46\% $\pm$ 0.48\% & 10.76\% $\pm$ 0.61\% & OOM                   & 20.21 $\pm$ 0.67\%                    \\
\cline{2-7}
  & \multirow{4}{*}{LoReTTA -- Cosine} & TGN   & 54.05\% $\pm$ 0.81\% & 15.42\% $\pm$ 1.24\% & 23.25\% $\pm$ 0.61\% & 20.30\% $\pm$ 0.42\% \\
  &                                     & JODIE & 28.15\% $\pm$ 0.54\% &  9.55\% $\pm$ 1.09\% & 11.78\% $\pm$ 0.62\% & 17.03\% $\pm$ 0.41\% \\
  &                                     & DYSAT & 48.39\% $\pm$ 1.66\% & 23.61\% $\pm$ 0.89\% & OOM                   & 17.00\% $\pm$ 0.11\% \\
  &                                     & TGAT  & 47.53\% $\pm$ 0.86\% & 11.63\% $\pm$ 0.47\% & OOM                   & 13.37\% $\pm$ 0.77\% \\
\cline{2-7}
  & \multirow{4}{*}{LoReTTA -- Jaccard} & TGN   & 55.18\% $\pm$ 0.42\% & 15.65\% $\pm$ 0.48\% & 22.91\% $\pm$ 0.70\% & 18.28\% $\pm$ 0.32\% \\
  &                                      & JODIE & 28.71\% $\pm$ 1.72\% &  9.16\% $\pm$ 0.37\% & 10.92\% $\pm$ 0.65\% & 15.64\% $\pm$ 0.72\% \\
  &                                      & DYSAT & 48.21\% $\pm$ 1.26\% & 22.42\% $\pm$ 1.01\% & OOM                   & 15.72\% $\pm$ 0.41\% \\
  &                                      & TGAT  & 49.25\% $\pm$ 0.69\% & 11.23\% $\pm$ 0.31\% & OOM                   & 13.33\% $\pm$ 0.50\%                    \\
\bottomrule
\end{tabular}
\caption{Attack performance of 4 TGNN models on 4 datasets against all attack methods.}      
\label{tab:attack-results}
\end{table*}

\subsection{Attack performance of all tested Sparsification heuristics of \ourmethod}
We show the performance of \ourmethod with 16 different sparsification strategies on 4 TGNN models and 4 datasets in Table~\ref{tab:attack-sparse-ablation}. In our sparsification framework, we compute the Temporal PageRank (TPR) vector $r^{(t)} \in \mathbb{R}^{|V|}$ at each timestamp $t$ to quantify evolving node influence. The temporal drift between two successive timestamps is measured by comparing $r^{(t)}$ and $r^{(t-1)}$ using a chosen similarity or divergence metric. Our empirical evaluation reveals that similarity-based metrics—particularly cosine similarity and Jaccard index—consistently outperform conventional distance-based metrics such as KL divergence, Jensen–Shannon (JS) divergence, and $\ell_1$ or $\ell_2$ norms. We now analyze the underlying reasons for this performance gap.

\paragraph{Sparsity and high dimensionality in TPR vectors.}
TPR vectors are typically high-dimensional and sparse since only a fraction of nodes exhibit temporal influence at each timestamp. In such regimes, distance-based metrics can become unstable or uninformative. For instance, KL and JS divergence assume overlapping supports and suffer when distributions contain zeros. In contrast, cosine and Jaccard similarity are tailored for sparse domains: cosine focuses on angular displacement (not magnitude), while Jaccard tracks support set overlap.

\paragraph{Directional sensitivity of cosine similarity.}
Cosine similarity captures shifts in influence direction across time:
\[
\cos\theta = \frac{\langle r^{(t)}, r^{(t-1)} \rangle}{\|r^{(t)}\| \cdot \|r^{(t-1)}\|}.
\]
This normalized, scale-invariant signal aligns with the model’s sensitivity to which nodes are gaining or losing importance, rather than their absolute rank. TGNNs are particularly vulnerable to such directional changes in influence dynamics, making cosine similarity an effective proxy for semantic drift.

\paragraph{Support tracking via Jaccard index.}
Jaccard similarity computes the ratio of shared active nodes between consecutive TPR vectors:
\[
\text{Jaccard}(A, B) = \frac{|A \cap B|}{|A \cup B|}, 
\]
\[\quad A = \text{support}(r^{(t)}),\]
\[B = \text{support}(r^{(t-1)})\]
This metric directly reflects reconfiguration in the set of influential nodes, which is highly correlated with performance drops in TGNNs trained on evolving interaction patterns. As sparsification aims to target such volatile moments, Jaccard serves as a reliable indicator of high-impact intervals.

\paragraph{Latent semantic alignment and adversarial leverage.}
Beyond their structural robustness, similarity metrics also operate in a space more aligned with the model's predictive behavior. Cosine similarity, in particular, indirectly reflects alignment in the TGNN’s latent embedding space—capturing semantic closeness between nodes. Removal of edges between such semantically aligned nodes degrades the coherence of learned prototypes, leading to impaired message passing and degraded representation learning.

\paragraph{Representation space is more fragile than observable structure.}
Our findings suggest that TGNNs are more sensitive to latent-space perturbations than to explicit topological alterations. While structural heuristics such as node degree or temporal frequency affect global statistics, similarity-based metrics destabilize the representation space in ways that are more difficult to recover from—especially under limited-knowledge or surrogate-free attack settings.

In sum, similarity-based metrics by virtue of their alignment with both sparsity patterns and semantic drift serve as more effective tools for perturbation planning in temporal graphs. Their ability to implicitly target the model’s latent structure enables stronger degradation with weaker assumptions, establishing them as superior alternatives to classical distance measures in the context of TGNN adversarial attacks.
\begin{table*}[htbp]
\centering
\small                          
\setlength{\tabcolsep}{4pt}     
\renewcommand{\arraystretch}{0.9}
\resizebox{\textwidth}{!}{%
\begin{tabular}{rlllll}
\toprule
\multicolumn{1}{l}{Attack} & Model & Wikipedia & UCI & MOOC & Enron \\
\midrule
\multirow{4}{*}{Degree}%
 & TGN   & 66.15\% $\pm$ 0.96\% & 17.51\% $\pm$ 1.39\% & 20.11\% $\pm$ 0.40\% & 12.31\% $\pm$ 0.68\% \\
 & JODIE & 31.35\% $\pm$ 1.22\% &  9.97\% $\pm$ 1.23\% & 11.46\% $\pm$ 0.35\% &  9.55\% $\pm$ 0.40\% \\
 & DySAT & 51.78\% $\pm$ 0.90\% & 19.65\% $\pm$ 1.33\% & --                    & 10.40\% $\pm$ 0.41\% \\
 & TGAT  & 54.48\% $\pm$ 0.65\% & 10.60\% $\pm$ 0.27\% & --                    & --                    \\
\cline{2-6}
\multirow{4}{*}{Jaccard}%
 & TGN   & --                    & 18.63\% $\pm$ 0.77\% & --                    & 19.28\% $\pm$ 0.51\% \\
 & JODIE & --                    & 10.27\% $\pm$ 0.90\% & --                    & 16.43\% $\pm$ 0.17\% \\
 & DySAT & --                    & 22.14\% $\pm$ 1.22\% & --                    & 17.46\% $\pm$ 0.14\% \\
 & TGAT  & --                    & 11.15\% $\pm$ 0.19\% & --                    & --                    \\
\cline{2-6}
\multirow{4}{*}{Pagerank} & TGN   & 66.51\% $\pm$ 1.22\% & 17.18\% $\pm$ 0.96\% & 19.75\% $\pm$ 0.77\% & 13.40\% $\pm$ 0.47\% \\
      & JODIE & 30.65\% $\pm$ 2.21\% & 9.00\% $\pm$ 0.59\% & 11.99\% $\pm$ 0.34\% & 10.82\% $\pm$ 0.23\% \\
      & DYSAT & 53.71\% $\pm$ 1.35\% & 19.60\% $\pm$ 1.14\% & -- & 11.16\% $\pm$ 0.34\% \\
      & TGAT  & 56.10\% $\pm$ 0.80\% & 11.05\% $\pm$ 0.59\% & -- & -- \\
\cline{2-6}
\multirow{4}{*}{Preference} & TGN   & 65.25\% $\pm$ 1.46\% & 17.42\% $\pm$ 1.61\% & 25.70\% $\pm$ 0.33\% & 14.25\% $\pm$ 0.39\% \\
      & JODIE & 27.85\% $\pm$ 3.76\% & 9.93\% $\pm$ 1.66\% & 12.28\% $\pm$ 0.73\% & 10.11\% $\pm$ 0.49\% \\
      & DYSAT & 52.75\% $\pm$ 1.50\% & 22.02\% $\pm$ 1.50\% & -- & 11.75\% $\pm$ 0.18\% \\
      & TGAT  & 55.10\% $\pm$ 0.26\% & 10.59\% $\pm$ 0.45\% & -- & -- \\
\cline{2-6}
\multirow{4}{*}{Random} & TGN   & 62.06\% $\pm$ 1.35\% & 17.40\% $\pm$ 1.49\% & 29.09\% $\pm$ 0.64\% & 27.46\% $\pm$ 0.69\% \\
      & JODIE & 24.18\% $\pm$ 0.72\% & 9.36\% $\pm$ 1.18\% & 12.36\% $\pm$ 0.38\% & 25.09\% $\pm$ 1.29\% \\
      & DYSAT & 48.14\% $\pm$ 1.68\% & 22.82\% $\pm$ 1.55\% & -- & 25.66\% $\pm$ 0.66\% \\
      & TGAT  & 55.38\% $\pm$ 1.12\% & 10.70\% $\pm$ 0.34\% & -- & -- \\
\cline{2-6}
\multirow{4}{*}{TER} & TGN   & 67.47\% $\pm$ 0.96\% & 15.46\% $\pm$ 1.06\% & 30.70\% $\pm$ 0.50\% & 19.00\% $\pm$ 0.53\% \\
      & JODIE & 33.86\% $\pm$ 0.56\% & 9.21\% $\pm$ 1.31\% & 14.43\% $\pm$ 0.57\% & 16.97\% $\pm$ 0.46\% \\
      & DYSAT & 52.88\% $\pm$ 1.84\% & 18.98\% $\pm$ 0.71\% & -- & 18.44\% $\pm$ 0.93\% \\
      & TGAT  & 61.46\% $\pm$ 0.48\% & 10.76\% $\pm$ 0.61\% & -- & -- \\
\cline{2-6}
\multirow{4}{*}{TPR-Chebyshev} & TGN   & 68.67\% $\pm$ 0.56\% & 18.05\% $\pm$ 0.83\% & 33.21\% $\pm$ 0.95\% & 20.18\% $\pm$ 1.08\% \\
      & JODIE & 41.38\% $\pm$ 1.69\% & 9.88\% $\pm$ 0.64\% & 16.44\% $\pm$ 0.62\% & 18.28\% $\pm$ 0.70\% \\
      & DYSAT & 53.28\% $\pm$ 1.47\% & 21.43\% $\pm$ 0.94\% & -- & 19.50\% $\pm$ 1.25\% \\
      & TGAT  & 63.50\% $\pm$ 0.72\% & 11.22\% $\pm$ 0.33\% & -- & -- \\
\cline{2-6}
\multirow{4}{*}{Combined-TER} & TGN   & 68.42\% $\pm$ 0.74\% & 17.40\% $\pm$ 1.03\% & 33.39\% $\pm$ 1.43\% & 21.36\% $\pm$ 0.41\% \\
      & JODIE & 40.71\% $\pm$ 2.41\% & 10.19\% $\pm$ 1.43\% & 16.58\% $\pm$ 0.63\% & 19.23\% $\pm$ 0.83\% \\
      & DYSAT & 53.82\% $\pm$ 1.21\% & 19.81\% $\pm$ 0.71\% & -- & 19.26\% $\pm$ 1.01\% \\
      & TGAT  & 63.70\% $\pm$ 0.64\% & 10.97\% $\pm$ 0.27\% & -- & -- \\
\cline{2-6}
\multirow{4}{*}{TPR-Cosine} & TGN   & 54.05\% $\pm$ 0.81\% & 15.42\% $\pm$ 1.24\% & 23.25\% $\pm$ 0.61\% & 20.30\% $\pm$ 0.42\% \\
      & JODIE & 28.15\% $\pm$ 0.54\% & 9.55\% $\pm$ 1.09\% & 11.78\% $\pm$ 0.62\% & 17.03\% $\pm$ 0.41\% \\
      & DYSAT & 48.39\% $\pm$ 1.66\% & 23.61\% $\pm$ 0.89\% & -- & 17.00\% $\pm$ 0.11\% \\
      & TGAT  & 47.53\% $\pm$ 0.86\% & 11.63\% $\pm$ 0.47\% & -- & 16.29\% $\pm$ 0.55\% \\
\cline{2-6}
\multirow{4}{*}{TPR-Euclidean} & TGN   & 68.63\% $\pm$ 1.08\% & 17.93\% $\pm$ 0.84\% & 33.23\% $\pm$ 0.80\% & 21.63\% $\pm$ 0.63\% \\
      & JODIE & 40.79\% $\pm$ 1.98\% & 10.83\% $\pm$ 1.43\% & 16.93\% $\pm$ 0.73\% & 18.21\% $\pm$ 1.08\% \\
      & DYSAT & 53.45\% $\pm$ 1.10\% & 22.95\% $\pm$ 1.19\% & -- & 18.84\% $\pm$ 0.79\% \\
      & TGAT  & 63.49\% $\pm$ 0.81\% & 11.21\% $\pm$ 0.44\% & -- & 24.31\% $\pm$ 0.50\% \\
\cline{2-6}
\multirow{4}{*}{TPR-Jaccard} & TGN   & 55.18\% $\pm$ 0.42\% & 15.65\% $\pm$ 0.48\% & 22.91\% $\pm$ 0.70\% & 18.28\% $\pm$ 0.32\% \\
      & JODIE & 28.71\% $\pm$ 1.72\% & 9.16\% $\pm$ 0.37\% & 10.92\% $\pm$ 0.65\% & 15.64\% $\pm$ 0.72\% \\
      & DYSAT & 48.21\% $\pm$ 1.26\% & 22.42\% $\pm$ 1.01\% & -- & 15.72\% $\pm$ 0.41\% \\
      & TGAT  & 49.25\% $\pm$ 0.69\% & 11.23\% $\pm$ 0.31\% & -- & -- \\
\cline{2-6}
\multirow{4}{*}{TPR-Jensen Shannon Divergence} & TGN   & 66.23\% $\pm$ 0.34\% & 19.06\% $\pm$ 1.31\% & 32.90\% $\pm$ 0.55\% & 16.61\% $\pm$ 0.70\% \\
      & JODIE & 34.19\% $\pm$ 1.14\% & 12.12\% $\pm$ 1.46\% & 15.89\% $\pm$ 0.87\% & 15.07\% $\pm$ 0.84\% \\
      & DYSAT & 50.65\% $\pm$ 0.53\% & 22.27\% $\pm$ 1.19\% & -- & 16.44\% $\pm$ 0.78\% \\
      & TGAT  & 61.40\% $\pm$ 0.56\% & 11.39\% $\pm$ 0.28\% & -- & -- \\
\cline{2-6}
\multirow{4}{*}{TPR-KL Divergence} & TGN   & 66.15\% $\pm$ 0.83\% & 19.93\% $\pm$ 1.71\% & 32.64\% $\pm$ 0.66\% & 16.73\% $\pm$ 0.73\% \\
      & JODIE & 35.61\% $\pm$ 1.36\% & 11.40\% $\pm$ 1.21\% & 15.06\% $\pm$ 1.16\% & 15.12\% $\pm$ 0.51\% \\
      & DYSAT & 51.52\% $\pm$ 1.30\% & 22.77\% $\pm$ 0.87\% & -- & 16.53\% $\pm$ 0.72\% \\
      & TGAT  & 61.24\% $\pm$ 0.87\% & 11.60\% $\pm$ 0.29\% & -- & -- \\
\cline{2-6}
\multirow{4}{*}{TPR-MSS} & TGN   & 68.70\% $\pm$ 1.00\% & 18.82\% $\pm$ 1.89\% & 33.30\% $\pm$ 0.78\% & 22.21\% $\pm$ 0.69\% \\
      & JODIE & 40.20\% $\pm$ 1.10\% & 10.97\% $\pm$ 1.46\% & 15.16\% $\pm$ 1.67\% & 19.00\% $\pm$ 0.85\% \\
      & DYSAT & 53.04\% $\pm$ 1.53\% & 22.98\% $\pm$ 1.04\% & -- & 19.12\% $\pm$ 1.67\% \\
      & TGAT  & 64.16\% $\pm$ 0.35\% & 11.59\% $\pm$ 0.25\% & -- & 24.28\% $\pm$ 0.66\% \\
\cline{2-6}
\multirow{4}{*}{TPR-$\text{MSS}^2$} & TGN   & 68.45\% $\pm$ 1.11\% & 18.64\% $\pm$ 1.16\% & 33.43\% $\pm$ 0.24\% & 21.69\% $\pm$ 0.76\% \\
      & JODIE & 40.27\% $\pm$ 1.93\% & 10.19\% $\pm$ 1.32\% & 15.88\% $\pm$ 2.13\% & 18.80\% $\pm$ 0.37\% \\
      & DYSAT & 52.78\% $\pm$ 0.81\% & 21.41\% $\pm$ 0.95\% & -- & 19.14\% $\pm$ 1.71\% \\
      & TGAT  & 63.50\% $\pm$ 0.59\% & 11.30\% $\pm$ 0.45\% & -- & -- \\
\cline{2-6}
\multirow{4}{*}{TPR-Wasserstein} & TGN   & 67.72\% $\pm$ 0.67\% & 17.74\% $\pm$ 0.70\% & 33.60\% $\pm$ 0.94\% & 21.71\% $\pm$ 0.58\% \\
      & JODIE & 39.28\% $\pm$ 1.66\% & 9.66\% $\pm$ 1.16\% & 15.87\% $\pm$ 2.28\% & 18.26\% $\pm$ 0.97\% \\
      & DYSAT & 53.28\% $\pm$ 1.32\% & 22.24\% $\pm$ 1.31\% & -- & 19.53\% $\pm$ 1.17\% \\
      & TGAT  & 63.70\% $\pm$ 0.56\% & 11.21\% $\pm$ 0.37\% & -- & -- \\
\bottomrule
\end{tabular}}
\caption{Test MRR (\%) of \ourmethod using various sparsification strategies across 4 datasets and 4 TGNNs.}
\label{tab:attack-sparse-ablation}
\end{table*}

\subsection{Compliance with C3 and C4 Constraints}
\label{constraint_analysis}

We empirically validate that our adversarial edge generation procedure respects the \textbf{C3 (recent node activity)} and \textbf{C4 (node degree preservation)} constraints, both of which are critical for stealth in continuous-time poisoning attacks.

\vspace{0.5em}
\noindent\textbf{C3: Temporal Activity Window.}  
Constraint C3 ensures that any adversarial edge connects nodes that were recently active prior to the assigned timestamp. To enforce this, our EdgeTimestamp Selector explicitly filters candidates based on temporal recency.  
Figure~\ref{fig:c3_constraint} shows the timestamp-node activity alignment on the Wikipedia dataset. All inserted edges lie within the expected window of recent activity, confirming our compliance with C3 and the temporal plausibility of generated interactions.

\begin{figure}[!ht]
    \centering
    \includegraphics[width=\columnwidth]{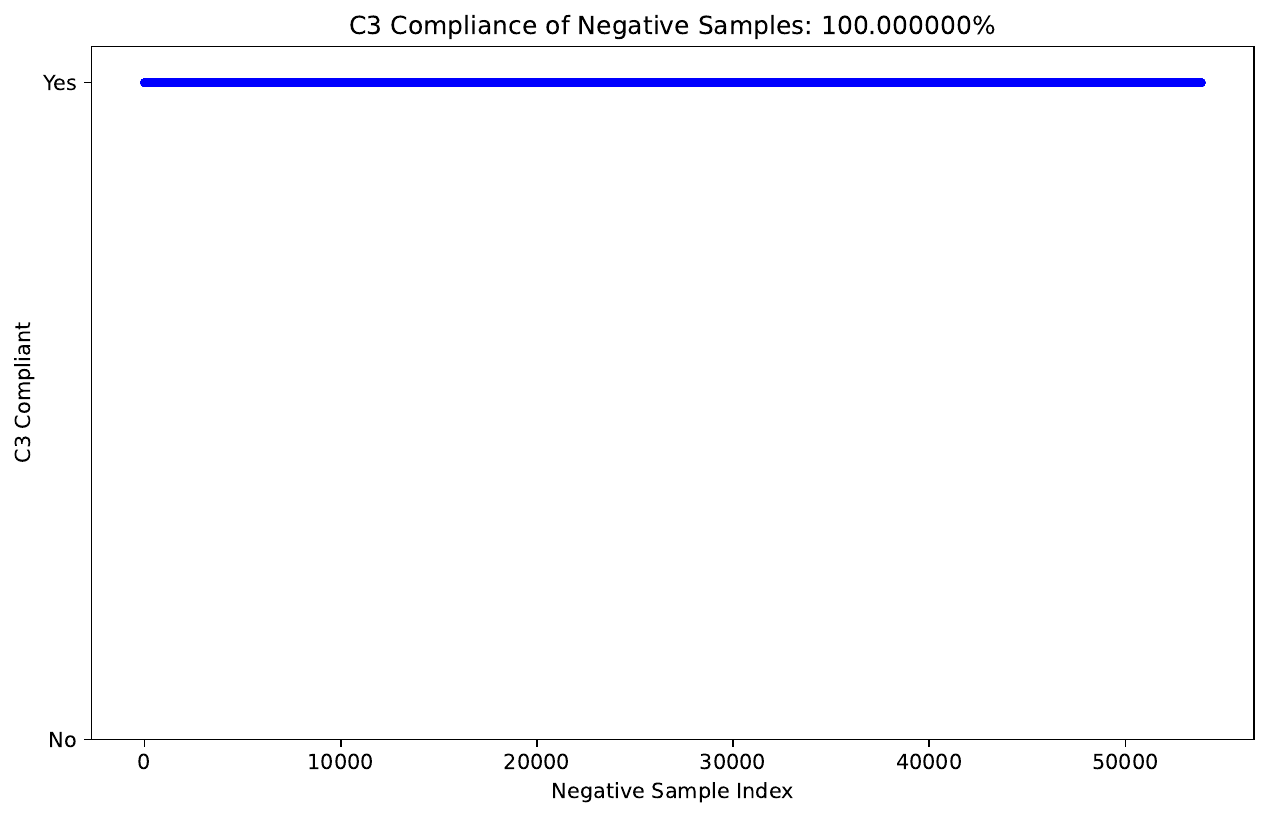}
    \caption{Compliance with the C3 constraint (Wikipedia dataset). Each adversarial edge connects nodes that were active in a fixed window prior to the sampled timestamp.}
    \label{fig:c3_constraint}
\end{figure}

\vspace{0.5em}
\noindent\textbf{C4: Degree Preservation.}  
Constraint C4 aims to preserve node degree distributions post-attack, thereby maintaining structural stealth. We track insertion and deletion counts per node, ensuring that adversarial edges compensate for removed ones in a degree-balanced fashion.

Figures~\ref{fig:bp_degree_dist} and~\ref{fig:non_bp_degree_dist} show degree distributions for bipartite and non-bipartite settings respectively. Both histograms and scatter plots confirm that regenerated node degrees closely mirror the original graph, validating the effectiveness of our degree-preserving strategy. In particular, the scatter plots reveal a near one-to-one match between original and modified degrees.

\begin{figure}[!t]
    \centering
    \begin{subfigure}[b]{\columnwidth}
        \centering
        \includegraphics[width=\columnwidth]{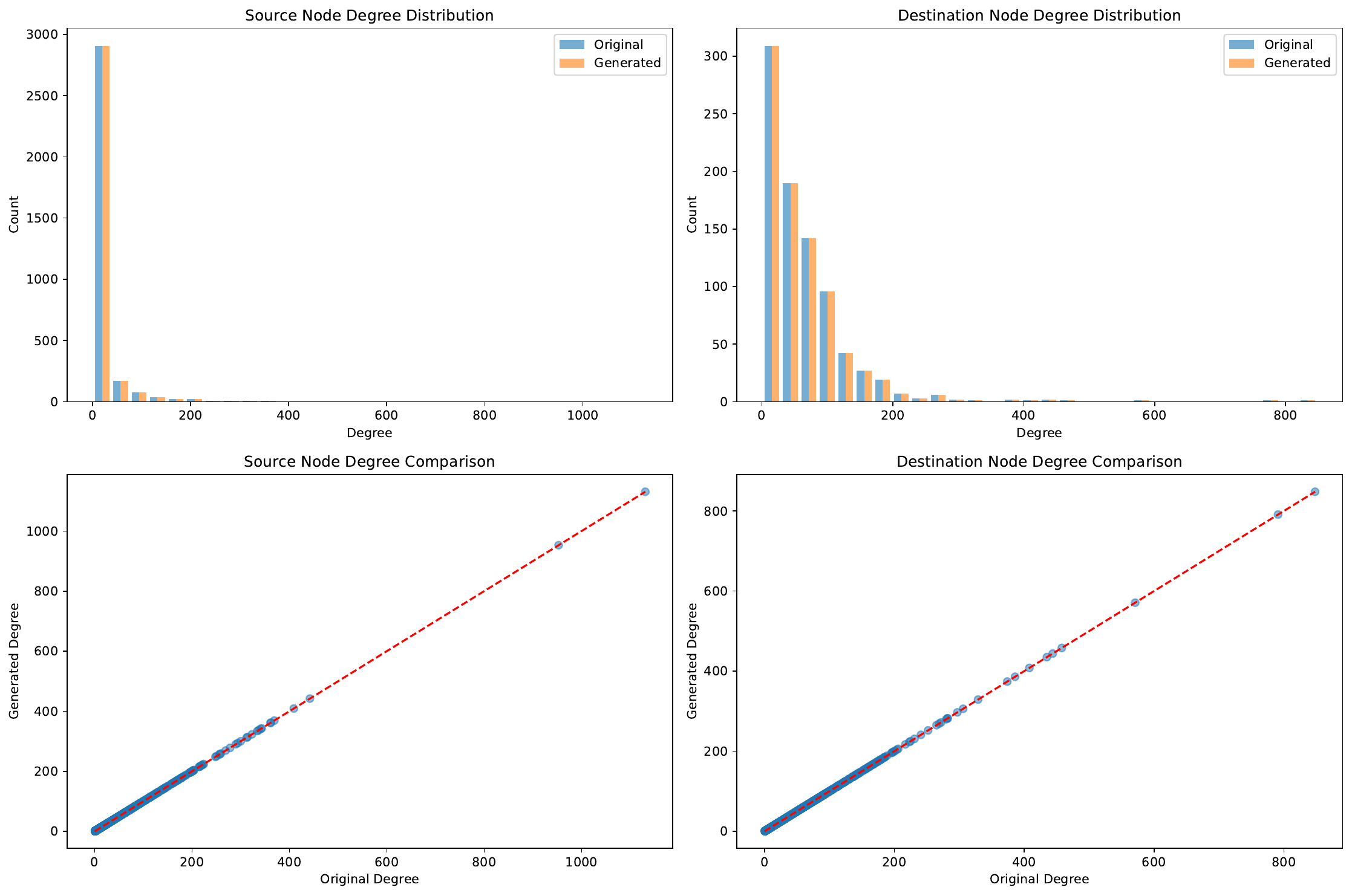}
        \caption{Bipartite dataset. Red dashes show original degrees; blue dots denote regenerated degrees. High overlap confirms C4 compliance.}
        \label{fig:bp_degree_dist}
    \end{subfigure}
    
    \vspace{0.5em}
    
    \begin{subfigure}[b]{\columnwidth}
        \centering
        \includegraphics[width=\columnwidth]{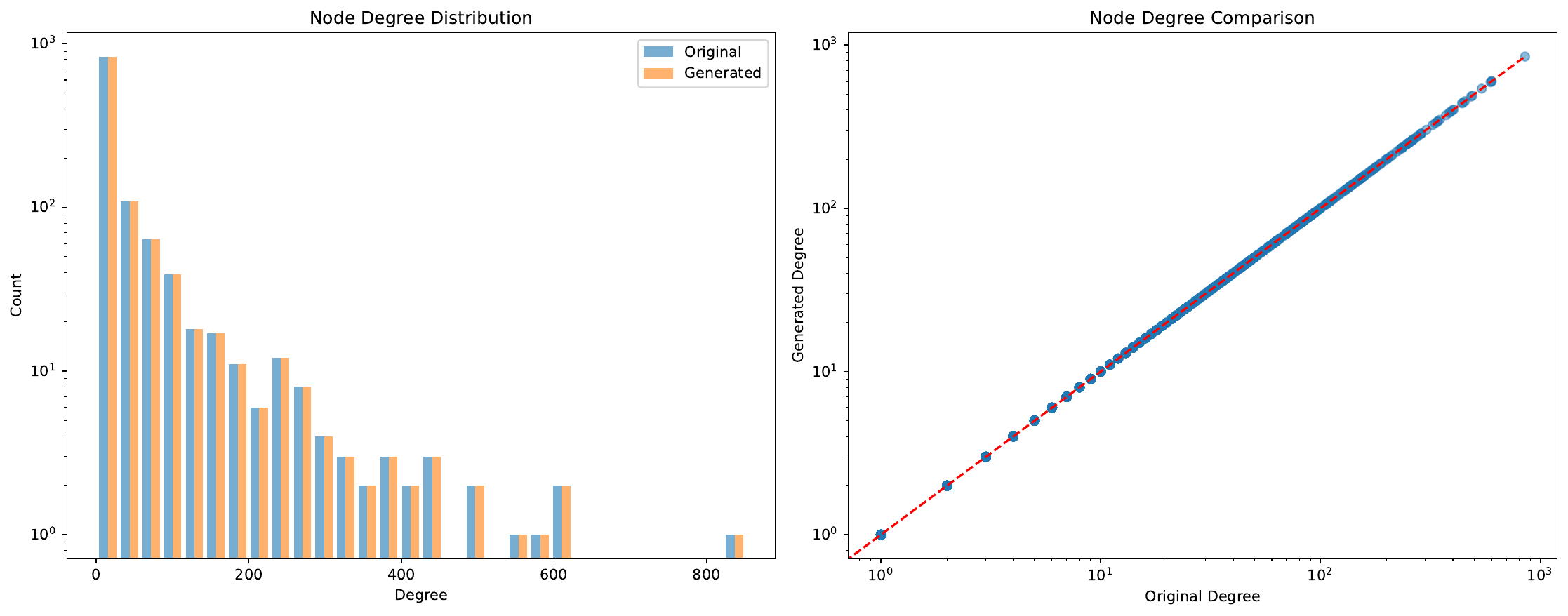}
        \caption{Non-bipartite dataset (UCI). Degree preservation is evident through strong alignment of regenerated and original degrees.}
        \label{fig:non_bp_degree_dist}
    \end{subfigure}
    
    \caption{C4 constraint validation across different graph types. Regenerated degrees preserve structural statistics, ensuring stealth.}
    \label{fig:c4_constraint}
\end{figure}

\vspace{0.5em}
\noindent\textbf{Takeaway.}  
Together, these results demonstrate that our negative sampling algorithm is \textit{constraint-aware by design}. It satisfies both temporal realism (C3) and structural stealth (C4), enabling a practical and undetectable attack strategy for CTDGs.

\subsection{Time and Space Complexity Analysis}\label{app:complexity}
We analyze both the theoretical and empirical complexity of \ourmethod, and show that it is significantly more efficient than T-SPEAR~\cite{lee2024spear}, the current state-of-the-art.

\subsubsection{Theoretical Time Complexity}
\begin{lemma}[Temporal PageRank Runtime]\label{lem:TPR}
Let \(G=(V,E)\) be a temporal graph, and \(\alpha, \beta \in [0,1]\) be the jump and decay parameters.
Algorithm \textsc{TPR} runs in 
\[
  O(|E| + |V|)\text{ time.}
\]
\end{lemma}

\begin{proof}
We maintain vectors \(r, s : V \to \mathbb{R}_{\ge 0}\) to track scores and intermediate contributions. 
Each edge \((u, v, t)\) is processed using a constant-time update map \(F_{\alpha,\beta}\), which requires $O(1)$ operations. 
Thus, processing all edges takes $O(|E|)$ time.

Afterward, a normalization pass over all vertices takes $O(|V|)$ time. 
Total cost: $O(|E| + |V|)$.
\end{proof}

\begin{lemma}[Timestamp Selector Runtime]\label{lem:TSS}
Let $d_{\max}$ be the maximum node degree and $k$ be the average number of candidate timestamps per edge.
The \textsc{Timestamp Selector} algorithm runs in
\[
  O\!\left(|V|(\log|V| + d_{\max}k\log k) + |E|\log|E|\right).
\]
\end{lemma}

\begin{proof}
\textbf{(i) Initialization:} Sorting vertices takes $O(|V|\log|V|)$; computing edge budgets takes $O(|E|)$.

\textbf{(ii) Edge-timestamp selection:} For each node \(v \in V\), we:
\begin{itemize}[noitemsep]
    \item Check capacity: $O(1)$,
    \item Collect candidate edges: $O(d_{\max})$,
    \item Select timestamps: $O(k\log k)$.
\end{itemize}
Total: $O(|V| d_{\max}k \log k)$.

\textbf{(iii) Recovery (if needed):} Sorting edges dominates with $O(|E|\log|E|)$.

Combining phases yields the final bound.
\end{proof}

\subsubsection{Empirical Time Complexity}
We evaluate runtime against T-SPEAR across four benchmark datasets using TGN as the base TGNN and a fixed 30\% perturbation rate.

\begin{table}[h]
  \centering
  \caption{Wall-clock time (in seconds) for \ourmethod and T-SPEAR. Lower is better.}
  \label{tab:runtime}
  \begin{tabular}{lcccc}
    \toprule
    \textbf{Method} & \textbf{Wikipedia} & \textbf{Enron} & \textbf{UCI} & \textbf{MOOC} \\
    \midrule
    T-SPEAR & 4299.02 & 2198.16 & 996.26  & 7320.78 \\
    \ourmethod &  370.18 &  620.30 & 434.36 & 1595.43 \\
    \bottomrule
  \end{tabular}
\end{table}

\noindent \ourmethod achieves up to a $10\times$ speedup over T-SPEAR, with an average of $3.91\times$ lower runtime across all datasets. This includes:
\begin{itemize}[noitemsep, leftmargin=1.5em]
    \item Surrogate model training (only for T-SPEAR),
    \item Edge addition (both),
    \item Edge removal (only for ~\ourmethod).
\end{itemize}

\noindent In contrast to T-SPEAR, which relies on training a learned surrogate model and computing model gradients for edge ranking, \ourmethod operates purely via heuristics over the interaction graph. This removes the need for backpropagation, reduces memory overhead, and enables black-box deployment. \textbf{As a result,~\ourmethod~is significantly more resource-efficient while maintaining strong attack performance.}
\vspace{0.5em}
\subsubsection{Space Complexity}

\begin{lemma}[TPR and TER Memory Footprint]\label{lem:spaceTPR}
Let $|V|$ be the number of nodes and $|T|$ the number of unique timestamps. The memory used by TPR and TER is
\[
  O(|V||T|).
\]
\end{lemma}

\begin{proof}
We store one PageRank vector \(r_t : V \to \mathbb{R}_{\ge 0}\) per timestamp \(t\), totaling \(|V||T|\) scalars.
An auxiliary degree-tracking array adds $O(|V|)$ overhead.
\end{proof}

\begin{lemma}[Timestamp Selector Memory Footprint]\label{lem:spaceTSS}
Let $W$ be the time window for candidate selection. The \textsc{Timestamp Selector} uses
\[
  O(|V| + |E|W)
\]
space.
\end{lemma}

\begin{proof}
Base graph storage requires $O(|V| + |E|)$.  
Each edge may generate up to $W$ candidates, resulting in $O(|E|W)$ auxiliary storage.
Additional counters per vertex require $O(|V|)$.
\end{proof}

\subsection{Implementation Details}
For all experiments, we utilized a constant learning rate of 0.0001 across the training of victim models, surrogate models, attack methods, and defense models. The Adam optimizer \cite{kingma2014adam} was employed with a batch size of 600, and all models were trained for a total of 100 epochs (50 epochs for DySAT). An early stopping criterion was applied after 50 epochs (25 epochs for DySAT), where training was halted if the validation Mean Reciprocal Rank (MRR) did not improve for 10 consecutive epochs.

\subsection{Hyperparameters for Models}
The models, including TGN \cite{rossi2020temporal}, JODIE \cite{kumar2019predicting}, TGAT \cite{xu2020inductive}, and DySAT \cite{sankar2020dysat}, were implemented using the Temporal Graph Learning\footnote{The TGL framework can be found at https://github.com/amazonresearch/tgl} (TGL) framework \cite{zhou2022tgl}. This framework provides optimized implementations, achieving better overall scores compared to the original implementations of the respective models. Consequently, we adopted TGL's default hyperparameter settings for all victim models across datasets. These settings include a learning rate of 0.0001, batch size of 600, and a hidden dimension of 100. A detailed summary of hyperparameters is provided in Table~\ref{table:victim_hyperparameters}.


\begin{table}[!ht]
\centering
\caption{Hyperparameter Settings for Victim Models}
\label{table:victim_hyperparameters}
\begin{tabular}{@{}lcc@{}}
\toprule
\textbf{Parameter} & \textbf{Value} & \textbf{Applicable Models} \\
\midrule
Learning Rate & 0.0001 & All \\
Batch Size & 600 & All \\
Hidden Dimension & 100 & All \\
Number of Epochs & 100 & All except DySAT \\
                 & 50  & DySAT \\
Early Stopping Patience & 10 epochs & All \\
\bottomrule
\end{tabular}
\end{table}

\subsection{Hardware Specifications and Environment}
All experiments are conducted on one NVIDIA A100 GPU (40GB memory), 512GB of RAM, and AMD EPYC 7742 64-Core Processors. Our models are implemented using PyTorch 1.12.1 \cite{paszke2019pytorch} and Deep Graph Library (DGL) 0.9.1 \cite{wang2019deep}.

\subsection{Datasets}\label{dataset-details}
We use four datasets, which are publicly available $\footnote{\url{https://zenodo.org/records/7213796#.Y1cO6y8r30o}}$: 
\begin{itemize}
    \item \textbf{Wikipedia:} This is a bipartite interaction graph representing edits made on Wikipedia pages over a one-month period. Nodes correspond to users and pages, while links denote editing activities with associated timestamps. This dataset contains 157,474 attributed interactions recorded over one month, involving 8,227 users and 1,000 pages. Each interaction represents a user editing a page, with the editing texts converted into LIWC-feature vectors.
    \item \textbf{MOOC:} This dataset represents a bipartite interaction network of online educational resources, where nodes correspond to students and course content units (e.g., videos and problem sets). It comprises interactions from 7,047 users with 97 items, resulting in 411,749 attributed interactions recorded over approximately one month. These interactions represent the access behavior of students to online course units.
    \item \textbf{Enron:} This dataset captures email communications among employees of the ENRON energy corporation over a three-year period. Nodes represent employees, and links denote email exchanges with timestamps. 
    \item \textbf{UCI:} This is an online communication network where nodes correspond to university students. Links represent messages exchanged between students, annotated with temporal information. It consists of 59,835 message interactions involving 1,899 unique users. 
\end{itemize}

\subsection{Dataset Statistics}
\label{sec:dataset_statistics}

We conduct our experiments on four widely used temporal graph datasets spanning both bipartite and non-bipartite domains. Table~\ref{tab:dataset_statistics} provides a summary.

\begin{table}[H]
\centering
\caption{Summary of datasets used in experiments. Each dataset contains timestamped interactions and supports Unix-level granularity.}
\label{tab:dataset_statistics}
\resizebox{\linewidth}{!}{
\begin{tabular}{lcccccc}
\toprule
\textbf{Dataset} & \textbf{Domain} & \textbf{Nodes} & \textbf{Edges} & \textbf{Bipartite} & \textbf{Duration} & \textbf{Granularity} \\
\midrule
Wikipedia & Social       & 9,227  & 157,474  & Yes   & 1 Month     & Unix Time \\
MOOC      & Interaction  & 7,144  & 411,749  & Yes   & 17 Months   & Unix Time \\
UCI       & Social       & 1,899  & 59,835   & No    & 196 Days    & Unix Time \\
Enron     & Communication & 184   & 125,235  & No    & 3 Years     & Unix Time \\
\bottomrule
\end{tabular}
}
\end{table}


\subsection{Baselines}
\label{baseline-details}

We benchmark \ourmethod~against a suite of heuristic baselines commonly used in link prediction and adversarial attack literature. These strategies generate perturbations either by adding new edges (used in prior works such as \citep{lee2024spear}) or by removing existing ones (used in our setting to highlight the effect of sparsification). Each method operates under the unnoticeability constraints outlined in~\S\ref{sec:unnoticeability_constraints}, and metrics are computed using a plain graph constructed from edges available just prior to perturbation.

\begin{itemize}
    \item \textbf{Random:} Adds or removes edges by uniformly sampling node pairs without replacement. Serves as a lower-bound baseline with no structural or temporal bias.
    
    \item \textbf{Preference:} Based on Preferential Attachment (PA)~\cite{liben2003link}, this method connects or deletes edges between nodes with the lowest PA score, where $\text{PA}(u, v) = |N(u)| \cdot |N(v)|$. Nodes with fewer connections are more likely to be perturbed, mimicking sparsity-aware behavior.
    
    \item \textbf{Jaccard:} Links or removes node pairs with the lowest Jaccard coefficient~\cite{liben2003link}, defined as 
    \[
    \text{Jaccard}(u, v) = \frac{|N(u) \cap N(v)|}{|N(u) \cup N(v)|}.
    \]
    This baseline is applied only to unipartite graphs, as common neighbor sets are undefined in bipartite settings.

    \item \textbf{Degree:} Inspired by Structack~\cite{hussain2021structack}, this strategy targets node pairs with minimal degree sum, i.e., $|N(u)| + |N(v)|$. It reflects a structural vulnerability perspective by focusing on low-connectivity regions.

    \item \textbf{PageRank:} Similar to Degree, this baseline perturbs edges between nodes with the lowest cumulative PageRank scores, $PR(u) + PR(v)$, targeting semantically weakly connected nodes.
\end{itemize}

For each baseline, we evaluate both edge addition~(ADD) and edge removal~(REM) variants: the former aligns with prior poisoning settings, while the latter is used in our context to demonstrate the effectiveness of edge sparsification. Together, these comparisons provide a comprehensive view of how simple graph heuristics fare against \ourmethod’s temporal and constraint-aware attack strategy. \textbf{In the original paper~\citep{lee2024spear}, the entire dataset—including validation and test sets—is poisoned. In our setup, we poison only the training set to ensure a fair comparison with our method. This difference in setup explains the variation in reported numbers for the T-SPEAR attack between our paper and the original.}

\subsection{Defense Methods}
These are the defense methods used.
\begin{enumerate}
    \item \textbf{TGN-SVD:} Drawing inspiration from GCN-SVD~\cite{entezari2020all}, this method constructs a plain graph using all edges from the training set and computes a low-rank approximation $\hat{A}$ of the adjacency matrix. The entries of $\hat{A}$ are used as weights in the loss function. For each edge $e = (u, v, t)$, the weight of the edge's loss in Eq.~(1) is determined by $\hat{a}_{uv}$.
    \item \textbf{TGN-COSINE:} Based on GNNGuard~\cite{zhang2020gnnguard}, this approach utilizes cosine similarity to filter edges. An edge $e = (u, v, t)$ is removed if the cosine similarity between the node embeddings $h_u(t)$ and $h_v(t)$ falls below a predefined threshold $\tau_{\text{cosine}}$.
    \item \textbf{TShield:} T-shield is proposed in Spear and Shield \cite{lee2024spear} and is one of the SotA defense methods for defense. It identifies and eliminates potential adversarial edges from the corrupted graphs, considering the evolution of dynamic graphs without prior knowledge of adversarial attacks.
    \item \textbf{TShield-F:} T-Shield-F is a variant of the method proposed in Spear and Shield~\cite{lee2024spear}, retaining only the edge-filtering step while excluding the temporal smoothing step.
\end{enumerate}

\subsubsection{Temporal Walk}\label{glossary}
In the context of temporal networks, a \textit{temporal walk} (also referred to as a \textit{time-respecting walk}) on a temporal graph $G$ is defined as an ordered sequence of edges: 
\[
(u_1, u_2, t_1), (u_2, u_3, t_2), \ldots, (u_j, u_{j+1}, t_j),
\]
where each edge $(u_i, u_{i+1}, t_i)$ connects nodes $u_i$ and $u_{i+1}$ at a timestamp $t_i$. A key property of a temporal walk is that the timestamps respect the order of the edges, i.e., $t_i \leq t_{i+1}$ for all $1 \leq i \leq j-1$. This ensures that the walk adheres to the temporal constraints of the network.

\subsection{On the Stealthiness of \textsc{\ourmethod}}
\label{sec:stealth_loretta}

Unlike node injection attacks that introduce new entities into the graph often detectable via anomaly detection or role analysis,~\textsc{\ourmethod} perturbs only interactions among existing, historically active nodes. This makes it significantly more stealthy. By adhering to unnoticeability constraints (C1–C4), the inserted edges mimic natural behavior in both temporal and structural ways. No new nodes are introduced, and edge insertions preserve degree and temporal-activity patterns, making them indistinguishable from genuine interactions(see~\ref{constraint_analysis}). This interaction-level perturbation enables practical and undetectable poisoning in real-world CTDGs, where node-level attacks would likely raise flags.

\subsection{Temporal PageRank and EdgeRank:Intuition, Formulation, and Computation}
\label{sec:Full_TPR}

\subsubsection{Temporal PageRank} Temporal PageRank (TPR) extends classical PageRank to temporal graphs by replacing static walks with time-respecting walks~\cite{rozenshtein2016temporal}. Let $Z^T(v, u \mid t)$ denote the set of temporal walks~(cf.~\S\ref{glossary}) from node $v$ to $u$ that occur strictly before time $t$. The probability of a walk $z$ is defined as:

\[
\Pr'[z \in Z^T(v, u \mid t)] = \frac{c(z \mid t)}{\sum\limits_{z' \in Z^T(v, x \mid t), \, x \in V, \, |z'| = |z|} c(z' \mid t)},
\]

where $c(z \mid t)$ is the decay-weighted count of $z$:

\[
c(z \mid t) = (1 - \beta)\!
\mathop{\prod_{\substack{((u_{i-1},u_i,t_i),\\(u_i,u_{i+1},t_{i+1}))\in z}}}
\beta^{\,\left|\left\{(u_i,y,t')\mid t'\in[t_i,t_{i+1}],\,y\in V\right\}\right|}.
\]

Transitions are penalized exponentially by the number of intervening interactions to model temporal decay. The final TPR score of node $u$ at time $t$ is:

\[
r(u, t) = \sum_{v \in V} \sum_{k=0}^{t} (1 - \alpha) \alpha^k \sum_{\substack{z \in Z_T(v, u \mid t) \\ |z| = k}} \Pr'[z \mid t],
\]

where $\alpha$ is the jump probability. This formulation biases the walk toward shorter, temporally coherent paths and naturally captures influence drift in evolving graphs.

\subsubsection{Temporal EdgeRank}\label{sec:TER}
To quantify the temporal importance of edges in the dynamic network, we introduce Temporal EdgeRank, which extends the concept of EdgeRank \cite{kucharczuk2022pagerank} to temporal graphs~(cf. Algorithm~\ref{alg:temporal-edgerank}). The algorithm first computes temporal PageRank scores for nodes using a damping factor $\beta$ and restart probability $\alpha$, capturing how node importance evolves over time. For each timestamp $t$, the algorithm identifies the set of active nodes and calculates the EdgeRank score by normalizing each node's temporal PageRank value by its overall outgoing degree. This normalization ensures that the influence of highly connected nodes is appropriately scaled. Specifically, the Temporal EdgeRank score at time $t$ is computed as the sum of the ratio between each active node's temporal PageRank score and its total outgoing degree. This approach effectively captures both the temporal dynamics of node importance and the structural properties of the network, providing a measure of edge significance that accounts for the temporal evolution of the graph structure. This approach integrates two complementary aspects of temporal graph analysis: (1) Temporal PageRank effectively captures the temporal dynamics of node importance by assigning higher weights to nodes participating in influential time-respecting walks, reflecting their centrality in evolving interactions. (2) Outgoing degree serves as a normalizing factor, quantifying the overall participation of a node across the temporal graph and preventing the measure from being biased solely by temporal activity bursts. Together, these components ensure that the significance of an edge reflects not just the structural properties of the graph at a single time point but also the temporal evolution of node roles and interactions over time.

\input{sections/TER-algo}

Together, TPR and TER enable fine-grained reasoning about temporal dynamics in node and edge importance. They serve as the foundation for sparsification in our attack pipeline, helping identify structurally sensitive and temporally volatile regions of the graph.

\subsubsection{Negative Sampling Algorithm}\label{sec:NS_algo}
The algorithm \ref{alg:temporal-edge-selection} employs a priority-based approach where nodes are first sorted based on their structural importance in the network. For each node in priority order, the algorithm selects valid edges (\textit{GetValidEdges}) within a specified time window while maintaining node capacity constraints. The selection process considers both topological constraints and temporal feasibility, ensuring that selected edges represent meaningful temporal interactions. For nodes in valid edges, we consider the temporally distant timestamp\footnote{We select a temporally distant timestamp for these nodes because it will be the least likely interaction for these nodes.} (\textit{SelectBest}). When the initial selection phase cannot fulfill the target edge budget, a recovery mechanism (\textit{RecoverEdges}) is activated to identify additional valid edge-timestamp pairs while preserving the temporal consistency of the network. This approach ensures both structural balance through capacity constraints and temporal coherence through window-based selection, making it particularly suitable for analyzing dynamic interaction patterns in temporal networks.
\input{sections/Selector-algo}

\label{A:NS_baselines}
\paragraph{Why can't other negative sampling algorithms be used instead of \ourmethod's:} Regarding negative sampling, random sampling (selecting edges arbitrarily out of temporal order) and algorithms like Havel-Hakimi~\cite{havelhakimi} are insufficient for the adversarial setting as they fail to account for stealth and temporal constraints. Random sampling lacks strategic impact, while Havel-Hakimi focuses on degree preservation without addressing temporal dynamics or unnoticeability requirements. TERA utilizes a novel backtracking algorithm to replace critical edges with carefully selected negative samples that preserve graph density and temporal structure, ensuring the perturbations remain subtle yet impactful. This combination of temporal PageRank-based edge selection and backtracking ensures an effective and generalizable adversarial attack framework tailored for temporal graph networks.

\subsection{Why is REM Attack on Wikipedia so Effective?}
\label{sec:wiki_sparse_effectiveness}

In our experiments, we observed that applying sparsification \textit{alone} on the Wikipedia dataset consistently outperforms combined strategies that include both sparsification and negative sampling (NS). We hypothesize that this is due to the unique structural and semantic properties of the Wikipedia graph.

\paragraph{Semantic Modularity \& Implicit Regularization}
Wikipedia exhibits a highly modular interaction structure, where users predominantly contribute to domain-specific clusters (e.g., nuclear physics vs. fashion). Sparsification selectively removes semantically meaningful edges within these modules, directly impairing the model’s ability to capture local coherence. In contrast, negative sampling often introduces edges between semantically distant nodes, implicitly injecting noise that violates this modularity. While such perturbations are adversarial in nature, they may inadvertently regularize the TGNN by encouraging cross-domain generalization, smoothing latent representations across otherwise disjoint clusters.

On semantically modular graphs like Wikipedia, sparsification alone is highly effective as it removes high-signal, intra-cluster edges critical to the model’s inference. In contrast, adding semantically inconsistent edges via NS can counteract this effect by implicitly encouraging broader generalization. This interplay likely explains why REM attacks on Wikipedia are more damaging than their \ourmethod counterparts.

\subsection{Ablation on Knowledge-Conditioned Sparsification}
\label{A:ablation_knowledge_appendix}

To analyze the behavior of our sparsification method under varying levels of adversary knowledge, we perform a controlled ablation by sparsifying the same dataset using progressively increasing knowledge thresholds: $k \in \{0.2, 0.4, 0.8\}$. For each threshold, we extract the resulting edge set and compute pairwise Jaccard similarity between all combinations. The resulting similarity matrix is visualized in Figure~\ref{fig:knowledge_abl_jaccard}.

\begin{figure}[t]
    \centering
    \includegraphics[width=0.85\linewidth]{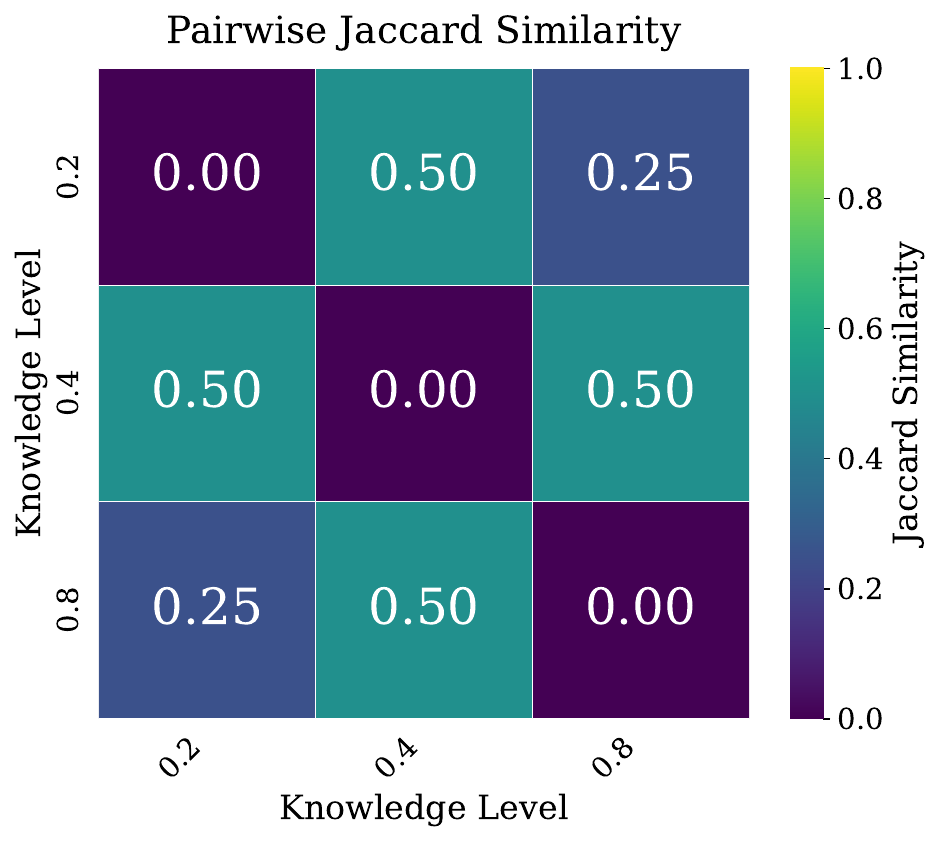}
    \caption{Pairwise Jaccard similarity between edge sets sparsified under different adversary knowledge levels ($k$). As knowledge increases, the edge sets diverge in a structured and deterministic manner.}
    \label{fig:knowledge_abl_jaccard}
\end{figure}

Our method exhibits three key properties that differentiate it from baseline random removal. First, it is \textit{deterministic}—the exact same edge set is produced for a given knowledge level across multiple runs, indicating no inherent randomness in the sparsification pipeline. This ensures reproducibility and enhances interpretability of downstream behavior. Second, the method displays \textit{monotonicity} in its edge removal: as the adversary's knowledge increases from $k_1$ to $k_2$ with $k_1 < k_2$, the edge set removed under $k_1$ is always a strict subset of the set removed under $k_2$. This is validated both by the decreasing Jaccard similarities and the subset relationships observed in our overlap analysis. Third, we observe a consistent trend of \textit{adversarial alignment}: higher knowledge levels lead to more aggressive removal of structurally important edges, suggesting that the method adapts meaningfully to the threat model. In contrast, randomly sparsified graphs show stable performance across varying knowledge levels, which we attribute to a non-monotonic relationship between edge removal and attack effectiveness. In such cases, only targeted perturbations cause significant degradation, while excessive or uninformed removals disrupt the attacker’s influence. This implies a “sweet spot” of adversary knowledge, where the attack is both precise and effective—too little misses key substructures, and too much reduces the coherence of the perturbation. Collectively, these observations highlight that our method produces more structured and predictable defense responses than random strategies.

\subsection{Does Cosine and T-shield Filtering remove true edges?}
\label{A:rob_hypothesis_ablation}
Table~\ref{tab:robustness_results_singlecol} shows the results of the average \% of adversarial edges removed during the filtering-based adversarial defense methods for Wikipedia and UCI on TGN when perturbed using \ourmethod at $p$ = 0.3. We can clearly observe that in most of the cases the \% of adversarial edges being removed is around 30\% which means many true edges are being removed which explains why after applying filtering based adversarial defense methods the performance is even worse than after \ourmethod.
\begin{table}[h]
\centering
\small
\setlength{\tabcolsep}{4pt}  
\begin{tabular}{|l|l|c|c|c|}
\hline
\textbf{Attack} & \textbf{Dataset} & \textbf{Cosine} & \textbf{T‑shield‑F} & \textbf{T‑shield} \\ \hline
\multirow{2}{*}{Degree}   & Wikipedia & 32.25 & 33.21 & 32.50 \\ 
                          & UCI       & 35.08 & 42.09 & 41.72 \\ \hline
\multirow{2}{*}{PageRank} & Wikipedia & 32.09 & 33.68 & 33.51 \\ 
                          & UCI       & 32.48 & 39.07 & 39.48 \\ \hline
\multirow{2}{*}{Cosine}   & Wikipedia & 35.92 & 30.95 & 30.85 \\ 
                          & UCI       & 29.46 & 29.18 & 28.77 \\ \hline
\multirow{2}{*}{Jaccard}  & Wikipedia & 38.44 & 30.70 & 30.77 \\ 
                          & UCI       & 30.72 & 30.86 & 30.98 \\ \hline
\multirow{2}{*}{TER}      & Wikipedia & 38.11 & 32.66 & 32.62 \\ 
                          & UCI       & 34.38 & 34.15 & 33.66 \\ \hline
\end{tabular}
\caption{Average \% of adversarial edges removed during filtering-based adversarial defense methods.}
\label{tab:robustness_results_singlecol}
\end{table}

\subsection{Threshold Free summary results of \ourmethod~against SotA Anomaly Detection methods}
\begin{table}[H]
\centering
\small
\setlength{\tabcolsep}{3pt} 
\renewcommand{\arraystretch}{0.9}
\begin{tabular}{@{}lcccc@{}}
\toprule
& \multicolumn{4}{c}{\textbf{Anomaly Detection Methods (AUPRC)}} \\
\cmidrule(lr){2-5}
\textbf{Attack} & \textbf{MIDAS} & \textbf{F-FADE} & \textbf{AnoEdge-L} & \textbf{AnoEdge-G} \\
\midrule
\multicolumn{5}{@{}l}{\textbf{Wikipedia Dataset}} \\
\midrule
Degree   & 0.44 & 0.43 & 0.48 & 0.45 \\
PageRank & 0.45 & 0.44 & 0.49 & 0.46 \\
Cosine   & 0.40 & 0.37 & 0.45 & 0.58 \\
Jaccard  & 0.32 & 0.35 & 0.43 & 0.54 \\
TER      & 0.50 & 0.34 & 0.40 & 0.34 \\
\midrule
\multicolumn{5}{@{}l}{\textbf{UCI Dataset}} \\
\midrule
Degree   & 0.37 & 0.55 & 0.24 & 0.31 \\
PageRank & 0.46 & 0.56 & 0.26 & 0.36 \\
Cosine   & 0.44 & 0.28 & 0.46 & 0.80 \\
Jaccard  & 0.48 & 0.24 & 0.44 & 0.83 \\
TER      & 0.72 & 0.19 & 0.43 & 0.56 \\
\bottomrule
\end{tabular}
\caption{Anomaly–class AUPRC (area under the precision–recall curve) of anomaly detection methods after perturbation using \ourmethod. }
\label{tab:anomaly_ablation}
\end{table}
Table~\ref{tab:anomaly_ablation} shows the resuls of AUPRC~(area under precision recall curve) score of anomaly detection methods after perturbation using \ourmethod~at $p$ = 0.3. We can see that the AUPRC is around 0.5 for most of cases except for 3 places in the UCI dataset. However, it must be noted that this score represents all the thresholds and it is impossible to determine which threshold to use in real-world making detecting adversarial edges even more challenging. After this, we can confidently say that \ourmethod is undetectable using SotA anomaly detection methods for edge streams.
\label{A:anomaly_ablation}

\subsection{Statistical Confidence Results}
\label{sec:sign-test}

For each \emph{Model–Dataset} pair, we compare every LoReTTA variant against all available baselines using a common, scale–free degradation metric in Table~\ref{tab:loretta_sign_all_pairs}.
Given the clean validation score $\text{Clean}_{m,d}$ for model $m$ on dataset $d$ and the post–attack score $\text{Score}_{m,d}$, we define the percentage drop
\[
\text{Drop}\% \;=\; \frac{\text{Clean}_{m,d} - \text{Score}_{m,d}}{\text{Clean}_{m,d}} \times 100,
\]
so that larger values indicate stronger attacks (greater degradation).
A negative value (rare) means the attack improved the score relative to clean.

\paragraph{Pairwise comparisons and outcomes.}
For a fixed \emph{Model–Dataset} pair, let $\mathcal{B}$ denote the set of baselines with valid scores.
For each LoReTTA variant $v$ and each $b\in\mathcal{B}$ we compare $\text{Drop}\%_v$ to $\text{Drop}\%_b$.
We record a \emph{win} if $\text{Drop}\%_v>\text{Drop}\%_b$, a \emph{loss} if $\text{Drop}\%_v<\text{Drop}\%_b$, and a \emph{tie} otherwise.
The table lists the counts of wins/losses/ties and the non-tied win rate.

\paragraph{Sign test over baselines.}
Because per-seed measurements are unavailable, we assess dominance using a one-sided exact binomial (sign) test across baselines:
\[
H_0:\; p\le \tfrac{1}{2}
\quad\text{vs.}\quad
H_1:\; p> \tfrac{1}{2},
\]
where $p$ is the probability that the LoReTTA variant wins against a randomly chosen baseline in $\mathcal{B}$ (ties are excluded).
Given $W$ wins and $L$ losses, the $p$-value is computed under $\mathrm{Binomial}(W{+}L,\tfrac{1}{2})$.
We report this $p$-value per row and mark significance with stars:
\textbf{*} $p{<}0.05$, \textbf{**} $p{<}0.01$, \textbf{***} $p{<}0.001$.
These stars summarize statistical significance (not effect size).

\paragraph{Interpretation and scope.}
Rows with \textbf{**}/\textbf{***} indicate that, for the given \emph{Model–Dataset}, the LoReTTA variant beats a clear majority of baselines with high confidence.
Rows without stars mean the evidence is insufficient to claim dominance at the 0.05 level.
All comparisons use the same \emph{Drop}\% metric, making results comparable across datasets with different absolute scales.

\onecolumn
\begingroup
\setlength{\tabcolsep}{4pt}
\scriptsize
\begin{longtable}{lllrrrrrrrc}
\caption{Sign-test results for LoReTTA variants vs.\ all baselines (higher Drop\% is stronger). One-sided exact binomial test across baselines per Model--Dataset; ties excluded.}
\label{tab:loretta_sign_all_pairs}\\
\toprule
Model & Dataset & Variant & Drop\% & Wins & Losses & Ties & Non-ties & Win rate & p-value & Sig.\\
\midrule
\endfirsthead
\toprule
Model & Dataset & Variant & Drop\% & Wins & Losses & Ties & Non-ties & Win rate & p-value & Sig.\\
\midrule
\endhead
\midrule
\multicolumn{11}{r}{\emph{Continued on next page}}\\
\midrule
\endfoot
\bottomrule
\endlastfoot
DySAT & Enron & Degree & 57.20 & 11 & 0 & 0 & 11 & 1.000 & 4.88e-04 & *** \\
DySAT & Enron & LoReTTA -- Cosine & 30.04 & 11 & 0 & 0 & 11 & 1.000 & 4.88e-04 & *** \\
DySAT & Enron & LoReTTA -- Jaccard & 35.31 & 11 & 0 & 0 & 11 & 1.000 & 4.88e-04 & *** \\
DySAT & Enron & LoReTTA -- TER & 24.12 & 11 & 0 & 0 & 11 & 1.000 & 4.88e-04 & *** \\
DySAT & Enron & Pagerank & 54.07 & 11 & 0 & 0 & 11 & 1.000 & 4.88e-04 & *** \\
DySAT & UCI & Degree & 72.44 & 11 & 0 & 0 & 11 & 1.000 & 4.88e-04 & *** \\
DySAT & UCI & LoReTTA -- Cosine & 66.89 & 11 & 0 & 0 & 11 & 1.000 & 4.88e-04 & *** \\
DySAT & UCI & LoReTTA -- Jaccard & 68.56 & 11 & 0 & 0 & 11 & 1.000 & 4.88e-04 & *** \\
DySAT & UCI & LoReTTA -- TER & 73.38 & 11 & 0 & 0 & 11 & 1.000 & 4.88e-04 & *** \\
DySAT & UCI & Pagerank & 72.51 & 11 & 0 & 0 & 11 & 1.000 & 4.88e-04 & *** \\
DySAT & Wikipedia & Degree & 22.25 & 5 & 4 & 0 & 9 & 0.556 & 5.00e-01 &  \\
DySAT & Wikipedia & LoReTTA -- Cosine & 27.34 & 5 & 4 & 0 & 9 & 0.556 & 5.00e-01 &  \\
DySAT & Wikipedia & LoReTTA -- Jaccard & 27.61 & 5 & 4 & 0 & 9 & 0.556 & 5.00e-01 &  \\
DySAT & Wikipedia & LoReTTA -- TER & 20.60 & 5 & 4 & 0 & 9 & 0.556 & 5.00e-01 &  \\
DySAT & Wikipedia & Pagerank & 19.35 & 5 & 4 & 0 & 9 & 0.556 & 5.00e-01 &  \\
JODIE & Enron & Degree & 56.59 & 11 & 0 & 0 & 11 & 1.000 & 4.88e-04 & *** \\
JODIE & Enron & LoReTTA -- Cosine & 22.59 & 11 & 0 & 0 & 11 & 1.000 & 4.88e-04 & *** \\
JODIE & Enron & LoReTTA -- Jaccard & 28.91 & 11 & 0 & 0 & 11 & 1.000 & 4.88e-04 & *** \\
JODIE & Enron & LoReTTA -- TER & 22.86 & 11 & 0 & 0 & 11 & 1.000 & 4.88e-04 & *** \\
JODIE & Enron & Pagerank & 50.82 & 11 & 0 & 0 & 11 & 1.000 & 4.88e-04 & *** \\
JODIE & MOOC & Degree & 69.36 & 9 & 0 & 0 & 9 & 1.000 & 1.95e-03 & ** \\
JODIE & MOOC & LoReTTA -- Cosine & 68.50 & 9 & 0 & 0 & 9 & 1.000 & 1.95e-03 & ** \\
JODIE & MOOC & LoReTTA -- Jaccard & 70.80 & 9 & 0 & 0 & 9 & 1.000 & 1.95e-03 & ** \\
JODIE & MOOC & LoReTTA -- TER & 61.42 & 3 & 6 & 0 & 9 & 0.333 & 9.10e-01 &  \\
JODIE & MOOC & Pagerank & 67.94 & 9 & 0 & 0 & 9 & 1.000 & 1.95e-03 & ** \\
JODIE & UCI & Degree & 59.96 & 9 & 2 & 0 & 11 & 0.818 & 3.27e-02 & * \\
JODIE & UCI & LoReTTA -- Cosine & 61.65 & 9 & 2 & 0 & 11 & 0.818 & 3.27e-02 & * \\
JODIE & UCI & LoReTTA -- Jaccard & 63.21 & 11 & 0 & 0 & 11 & 1.000 & 4.88e-04 & *** \\
JODIE & UCI & LoReTTA -- TER & 63.01 & 11 & 0 & 0 & 11 & 1.000 & 4.88e-04 & *** \\
JODIE & UCI & Pagerank & 63.86 & 11 & 0 & 0 & 11 & 1.000 & 4.88e-04 & *** \\
JODIE & Wikipedia & Degree & 50.40 & 6 & 3 & 0 & 9 & 0.667 & 2.54e-01 &  \\
JODIE & Wikipedia & LoReTTA -- Cosine & 55.46 & 7 & 2 & 0 & 9 & 0.778 & 8.98e-02 &  \\
JODIE & Wikipedia & LoReTTA -- Jaccard & 54.57 & 7 & 2 & 0 & 9 & 0.778 & 8.98e-02 &  \\
JODIE & Wikipedia & LoReTTA -- TER & 46.42 & 6 & 3 & 0 & 9 & 0.667 & 2.54e-01 &  \\
JODIE & Wikipedia & Pagerank & 51.50 & 6 & 3 & 0 & 9 & 0.667 & 2.54e-01 &  \\
TGAT & Enron & Degree & 58.22 & 11 & 0 & 0 & 11 & 1.000 & 4.88e-04 & *** \\
TGAT & Enron & LoReTTA -- Cosine & 48.18 & 11 & 0 & 0 & 11 & 1.000 & 4.88e-04 & *** \\
TGAT & Enron & LoReTTA -- Jaccard & 48.33 & 11 & 0 & 0 & 11 & 1.000 & 4.88e-04 & *** \\
TGAT & Enron & LoReTTA -- TER & 21.67 & 11 & 0 & 0 & 11 & 1.000 & 4.88e-04 & *** \\
TGAT & Enron & Pagerank & 57.95 & 11 & 0 & 0 & 11 & 1.000 & 4.88e-04 & *** \\
TGAT & UCI & Degree & 34.97 & 11 & 0 & 0 & 11 & 1.000 & 4.88e-04 & *** \\
TGAT & UCI & LoReTTA -- Cosine & 28.65 & 10 & 1 & 0 & 11 & 0.909 & 5.86e-03 & ** \\
TGAT & UCI & LoReTTA -- Jaccard & 31.10 & 10 & 1 & 0 & 11 & 0.909 & 5.86e-03 & ** \\
TGAT & UCI & LoReTTA -- TER & 33.99 & 11 & 0 & 0 & 11 & 1.000 & 4.88e-04 & *** \\
TGAT & UCI & Pagerank & 32.21 & 10 & 1 & 0 & 11 & 0.909 & 5.86e-03 & ** \\
TGAT & Wikipedia & Degree & 4.42 & 4 & 5 & 0 & 9 & 0.444 & 7.46e-01 &  \\
TGAT & Wikipedia & LoReTTA -- Cosine & 16.61 & 4 & 5 & 0 & 9 & 0.444 & 7.46e-01 &  \\
TGAT & Wikipedia & LoReTTA -- Jaccard & 13.60 & 4 & 5 & 0 & 9 & 0.444 & 7.46e-01 &  \\
TGAT & Wikipedia & LoReTTA -- TER & -7.82 & 1 & 8 & 0 & 9 & 0.111 & 9.98e-01 &  \\
TGAT & Wikipedia & Pagerank & 1.58 & 2 & 7 & 0 & 9 & 0.222 & 9.80e-01 &  \\
TGN & Enron & Degree & 55.88 & 11 & 0 & 0 & 11 & 1.000 & 4.88e-04 & *** \\
TGN & Enron & LoReTTA -- Cosine & 27.24 & 11 & 0 & 0 & 11 & 1.000 & 4.88e-04 & *** \\
TGN & Enron & LoReTTA -- Jaccard & 34.48 & 11 & 0 & 0 & 11 & 1.000 & 4.88e-04 & *** \\
TGN & Enron & LoReTTA -- TER & 31.90 & 11 & 0 & 0 & 11 & 1.000 & 4.88e-04 & *** \\
TGN & Enron & Pagerank & 51.97 & 11 & 0 & 0 & 11 & 1.000 & 4.88e-04 & *** \\
TGN & MOOC & Degree & 67.46 & 9 & 0 & 0 & 9 & 1.000 & 1.95e-03 & ** \\
TGN & MOOC & LoReTTA -- Cosine & 62.38 & 9 & 0 & 0 & 9 & 1.000 & 1.95e-03 & ** \\
TGN & MOOC & LoReTTA -- Jaccard & 62.93 & 9 & 0 & 0 & 9 & 1.000 & 1.95e-03 & ** \\
TGN & MOOC & LoReTTA -- TER & 50.32 & 8 & 1 & 0 & 9 & 0.889 & 1.95e-02 & * \\
TGN & MOOC & Pagerank & 68.04 & 9 & 0 & 0 & 9 & 1.000 & 1.95e-03 & ** \\
TGN & UCI & Degree & 60.38 & 11 & 0 & 0 & 11 & 1.000 & 4.88e-04 & *** \\
TGN & UCI & LoReTTA -- Cosine & 65.11 & 11 & 0 & 0 & 11 & 1.000 & 4.88e-04 & *** \\
TGN & UCI & LoReTTA -- Jaccard & 64.59 & 11 & 0 & 0 & 11 & 1.000 & 4.88e-04 & *** \\
TGN & UCI & LoReTTA -- TER & 65.02 & 11 & 0 & 0 & 11 & 1.000 & 4.88e-04 & *** \\
TGN & UCI & Pagerank & 61.13 & 11 & 0 & 0 & 11 & 1.000 & 4.88e-04 & *** \\
TGN & Wikipedia & Degree & 17.83 & 5 & 4 & 0 & 9 & 0.556 & 5.00e-01 &  \\
TGN & Wikipedia & LoReTTA -- Cosine & 32.86 & 6 & 3 & 0 & 9 & 0.667 & 2.54e-01 &  \\
TGN & Wikipedia & LoReTTA -- Jaccard & 31.45 & 6 & 3 & 0 & 9 & 0.667 & 2.54e-01 &  \\
TGN & Wikipedia & LoReTTA -- TER & 16.19 & 5 & 4 & 0 & 9 & 0.556 & 5.00e-01 &  \\
TGN & Wikipedia & Pagerank & 17.38 & 5 & 4 & 0 & 9 & 0.556 & 5.00e-01 &  \\
\end{longtable}
\endgroup
\twocolumn

\subsection{Effect of Time Window on C3 and C4 Constraints}
\label{app:effect_of_time_window}

Figure~\ref{fig:c3_400s} highlights that the 400s time window satisfies the C3 compliance condition. In contrast, the 1800s window fails to maintain this distribution (Figure~\ref{fig:c3_1800s}).

However, the opposite trend appears in the C4 constraint. The 1800s window exhibits better degree preservation (Figure~\ref{fig:c4_1800s}), whereas the 400s window shows notable deviation (Figure~\ref{fig:c4_400s}). This is likely due to sparsity in the 400s, where timestamps can have zero active nodes because of smaller time-window.

Our main experiments use the time-window, guided by \citet{lee2024spear} on TGNN degradation.

\input{sections/sec-7-ImpactStatement}
\begin{figure}[!t]
    \centering
    \includegraphics[width=\linewidth]{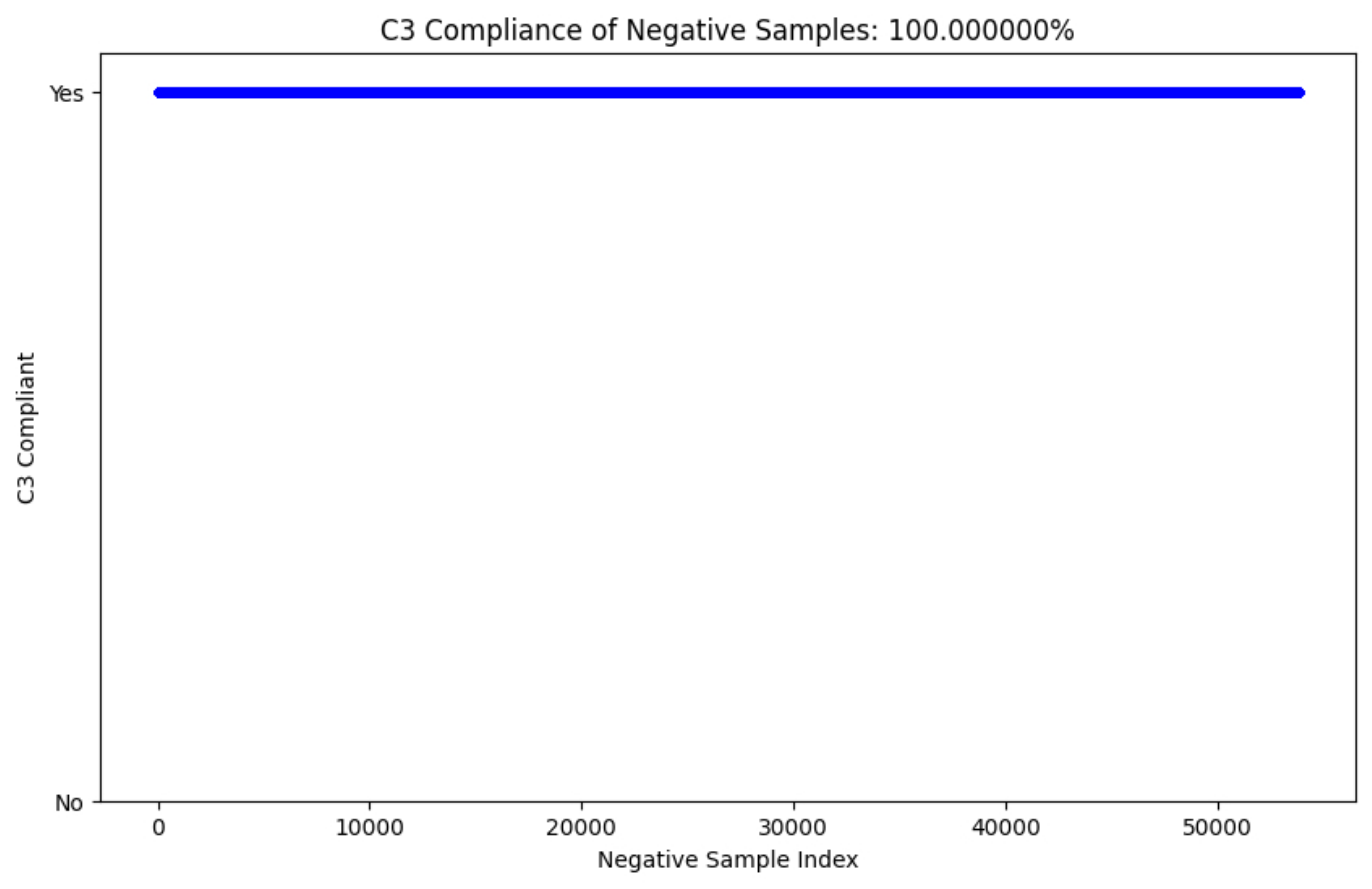}
    \caption{C3 compliance for the 400s dataset. All negative samples satisfy the C3 constraint, indicated by blue dots.}

    \label{fig:c3_400s}
\end{figure}

\begin{figure}[!t]
    \centering
    \includegraphics[width=\linewidth]{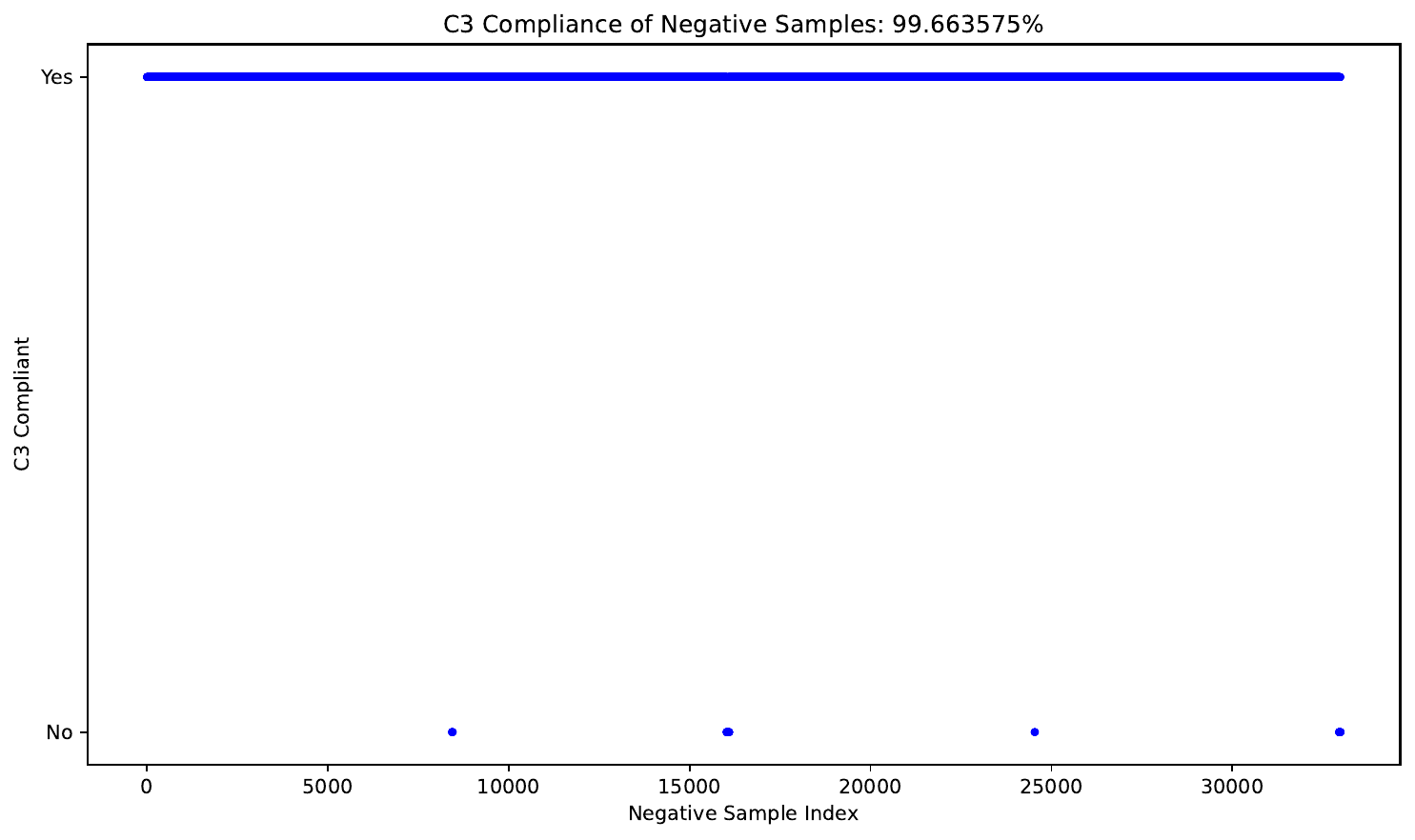}
    \caption{C3 compliance for the 1800s dataset. Blue dots at top indicate compliant samples; a small fraction at bottom BC
    violate the C3 constraint.}
    \label{fig:c3_1800s}
\end{figure}

\begin{figure}[!t]
    \centering
    \includegraphics[width=\linewidth]{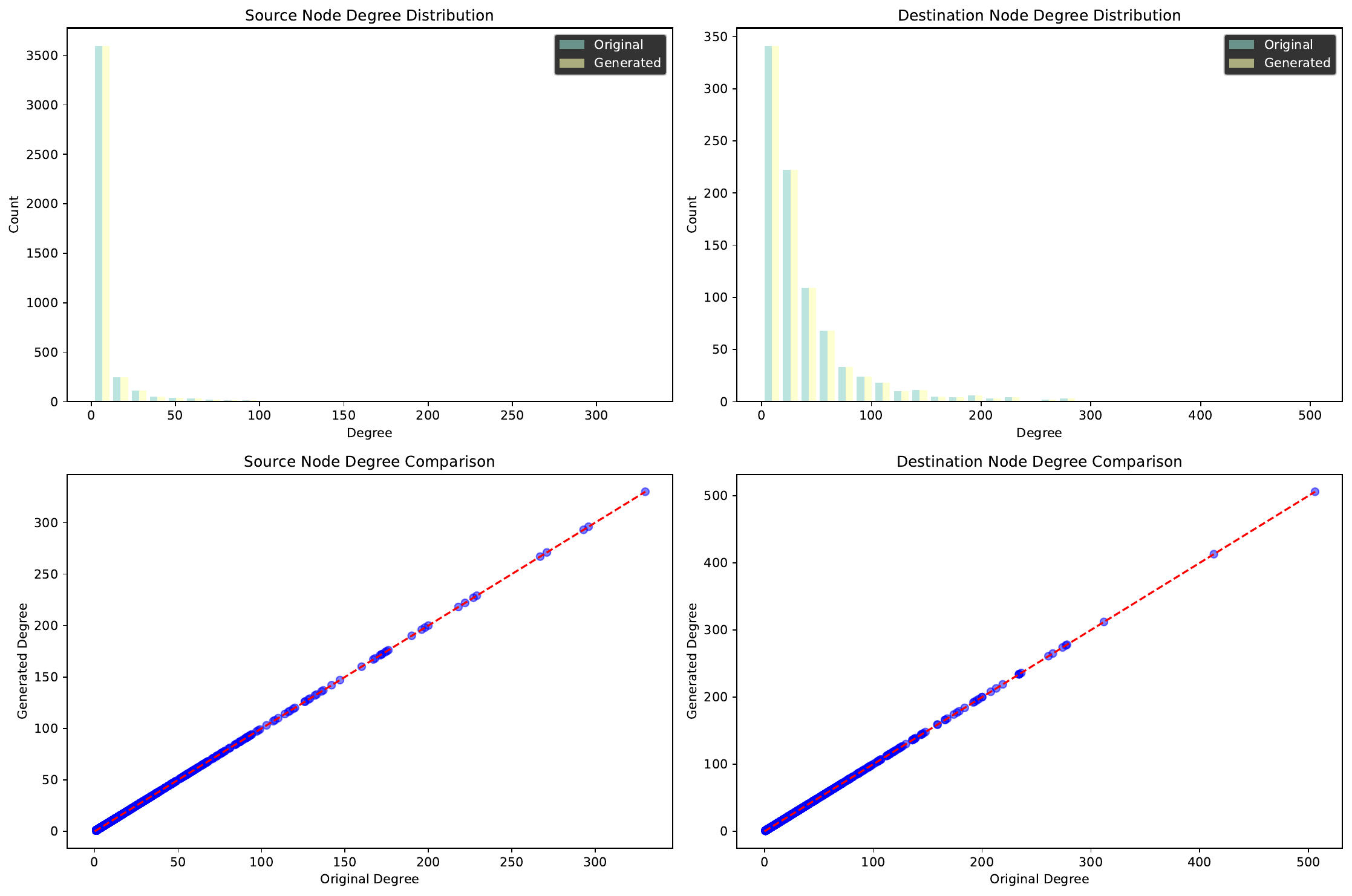}
    \caption{C4 compliance (degree preservation) for 1800s.  Scatter alignment confirms statistical matching. Original vs regenerated degrees are identical.}
    \label{fig:c4_1800s}
\end{figure}

\begin{figure}[!t]
    \centering
    \includegraphics[width=\linewidth]{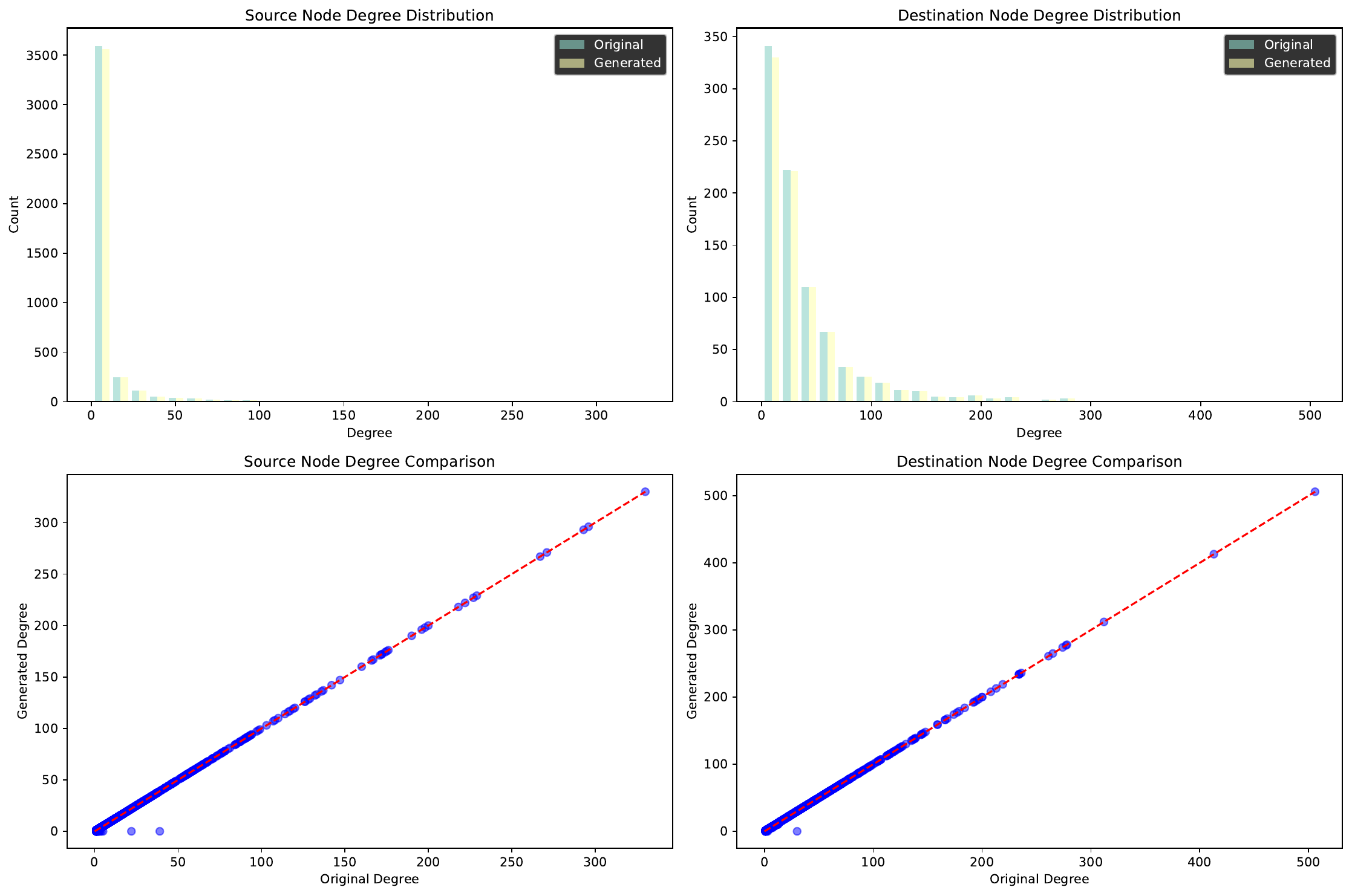}
    \caption{C4 compliance for the 400s dataset. While overall degree preservation holds, early timestamps show deviation due to low activity. Many nodes were yet to be created, reducing the likelihood of matching active degrees.
}
    \label{fig:c4_400s}
\end{figure}

\input{sections/reproducibilty_checklist.incl}

%% file: sections/TER-algo.tex
\begin{algorithm}[!htb]
   \caption{Temporal EdgeRank}
   \label{alg:temporal-edgerank}
\begin{algorithmic}
   \STATE {\bfseries Input:} Edge set $E$ as $(u,v,t)$ tuples, damping $\beta$, restart prob. $\alpha$
   \STATE {\bfseries Output:} Temporal EdgeRank scores
   \STATE $\text{TPR} \leftarrow \text{TemporalPageRank}(E, \beta, \alpha)$
   \STATE $\text{out\_deg} \leftarrow \text{ComputeOutDegree}(E)$
   \STATE $TER \leftarrow \{\}$  \{Temporal EdgeRank\}
   \FORALL{timestamp $t$ in $E$}
      \STATE $nodes_t \leftarrow \{u \mid (u,v,t) \in E\}$
      \STATE $TER[t] \leftarrow \sum\limits_{n \in nodes_t} \frac{TPR[n,t]}{out\_deg[n]}$
   \ENDFOR
   \STATE \textbf{return} $TER$
\end{algorithmic}
\end{algorithm}

%% file: sections/Selector-algo.tex
   
   

\begin{algorithm}[!htb]
   \caption{Timestamp Selector (Main Procedure)}
   \label{alg:temporal-edge-selection}
\begin{algorithmic}[1]
   \STATE {\bfseries Input:} Graph $G(V,E)$, Node capacity $C$, Time window $W$
   \STATE {\bfseries Output:} Edge–timestamp pairs $M$
   \STATE $M \leftarrow \emptyset$ \hfill \COMMENT{Final edge–timestamp assignments}
   \STATE $B \leftarrow$ \Call{ComputeBudget}{$C$} \hfill \COMMENT{Total number of insertions}
   \STATE Sort nodes $v \in V$ by structural priority
   \FORALL{$v \in V$}
      \WHILE{\Call{HasCapacity}{$v$} \textbf{and} $|M| < B$}
         \STATE $E_c \leftarrow$ \Call{GetValidEdges}{$v, W$}
         \IF{$E_c \neq \emptyset$}
            \STATE $(e, t) \leftarrow$ \Call{SelectBest}{$E_c$}
            \STATE $M \leftarrow M \cup \{(e, t)\}$
            \STATE \Call{UpdateCapacity}{$e$}
         \ENDIF
      \ENDWHILE
   \ENDFOR
   \IF{$|M| < B$}
      \STATE $M \leftarrow M \cup$ \Call{RecoverEdges}{$B - |M|$}
   \ENDIF
   \STATE \textbf{return} $M$
\end{algorithmic}
\end{algorithm}

\vspace{0.5em}
\noindent\textbf{Helper Routines.}
\begin{itemize}[leftmargin=1.5em, itemsep=0.3em]
    \item \textsc{ComputeBudget}($C$): Calculates the total number of edges to be added, based on node-level capacity $C$.
    \item \textsc{HasCapacity}($v$): Returns true if node $v$ has not exceeded its insertion/deletion budget.
    \item \textsc{GetValidEdges}($v, W$): Returns a set of candidate edges involving $v$ within the time window $W$, filtered by feasibility constraints (e.g., C3/C4).
    \item \textsc{SelectBest}($E_c$): Chooses the most adversarially valuable edge–timestamp pair from the candidate set, based on sparsification metrics.
    \item \textsc{UpdateCapacity}($e$): Updates the capacity counters for nodes in edge $e$.
    \item \textsc{RecoverEdges}($k$): Backup strategy to find additional $k$ edge–timestamp pairs when primary routine falls short (e.g., due to constraint violations).
\end{itemize}


%% file: sections/sec-7-ImpactStatement.tex
\section*{Ethical Considerations}
This work presents a novel adversarial attack on TGNNs, aiming to identify vulnerabilities in temporal graph learning models. While adversarial attacks have the potential to be misused, our primary motivation is to improve the robustness of temporal graph-based models by identifying weaknesses that could be exploited in malicious scenarios. By understanding these vulnerabilities, we hope to enable the development of defense mechanisms and inspire more resilient temporal graph architectures.

This work could benefit areas such as network security, fraud detection, and dynamic system modeling by highlighting risks and promoting research on robust temporal graph methods. However, the methods described in this paper, if misapplied, could be used to undermine real-world systems relying on temporal graphs. We strongly advocate for ethical use of this research, emphasizing its role in advancing security and resilience in machine learning.

As with all adversarial research, this work underscores the importance of ethical responsibility in deploying machine learning models in safety-critical applications. Future research could focus on leveraging insights from this work to develop robust defenses against such attacks and to minimize the risks associated with deploying temporal graph models in sensitive environments.